\documentclass[10pt,twocolumn]{article}

\usepackage[margin=1in]{geometry}
\usepackage{amsmath,amssymb,amsthm}
\usepackage{mathtools}
\usepackage{booktabs}
\usepackage{graphicx}
\usepackage[dvipsnames,table]{xcolor}
\usepackage{hyperref}
\usepackage{cleveref}
\usepackage{microtype}
\usepackage{algorithm}
\usepackage{algpseudocode}
\usepackage{multirow}
\usepackage{subcaption}
\usepackage{natbib}
\usepackage{xurl}

\newtheorem{proposition}{Proposition}
\newtheorem{corollary}{Corollary}[proposition]
\newtheorem{remark}{Remark}

\newcommand{\Xhist}{X_{\mathrm{hist}}}
\newcommand{\Aeff}{A_{\mathrm{eff}}}
\newcommand{\Atil}{\tilde{A}}
\newcommand{\Btil}{\tilde{B}}
\newcommand{\dz}{d_z}
\newcommand{\FMamKO}{\mathcal{F}_{\mathrm{MamKO}}}
\newcommand{\FBL}{\mathcal{F}_{\mathrm{BL}}}
\newcommand{\E}{\mathbb{E}}
\newcommand{\R}{\mathbb{R}}

\title{Bilinear Mamba-Koopman Neural MPC for Varying Dynamics}

\author{
  Matan Pagi \quad Zohar Sorek \\
  Psistar AI \\
  \texttt{\{matan,zohar\}@psistar.ai}
}

\date{}

\begin{document}
\maketitle

\begin{abstract}
Koopman-based neural MPC models generate time-varying dynamics from
historical data, but preserve convexity by enforcing that the system
operator is independent of the current control input. This conditional
independence constraint limits adaptation to changing dynamics within a
single MPC horizon, particularly under time-varying conditions and under
stale-plan execution.

We propose \emph{Bilinear Mamba-Koopman Neural MPC}, a minimal extension that
introduces control-dependent coupling in the latent dynamics, allowing the
effective operator to adapt to the current input. The resulting model is a
strict generalization of the standard linear, conditional-independence formulation, adds less than $1\%$
parameters, and admits exact model Jacobians
that enable efficient Sequential Convex Programming~(SCP) with
monotone-descent and KKT convergence results under standard trust-region assumptions.

Across CartPole and RSCP benchmarks in time-invariant and time-varying
regimes, the proposed model matches or improves forecasting accuracy on every
cell when training noise is averaged out, with strict gains where
control-state coupling is structurally present. Its main closed-loop
gains appear in the RSCP TV task, where iterative SCP improves
 adaptation within the horizon and the bilinear model substantially stabilizes training; in
CartPole TV, the gains are modest but consistent. In delayed re-planning
experiments on the time-varying variants, the bilinear model degrades more
gracefully under stale-plan execution, maintaining a consistent advantage
on CartPole TV and a substantially larger robustness margin on RSCP TV.
These results show that control-dependent latent dynamics provides a simple
and effective mechanism for robust MPC under varying conditions.
\end{abstract}


\section{Introduction}
\label{sec:intro}

The Koopman operator offers an elegant reformulation of nonlinear dynamics:
any nonlinear system, however complex in state space, admits a \emph{linear}
representation in an infinite-dimensional function
space~\citep{koopman1931hamiltonian,mezic2005spectral}. Finite-dimensional
approximations learned from trajectory data have enabled controllers that
combine the tractability of linear MPC with operation on genuinely nonlinear
physical systems~\citep{korda2018linear,lusch2018deep,williams2015data}. For real-time industrial control---chemical reactors, compressors, power
converters---tractable convex optimization at each control step is a hard
operational requirement.

Recent work has extended Koopman learning to \emph{time-varying} systems by
replacing fixed operator matrices with sequences generated from historical
data. A representative instance of this class is the Mamba-based Koopman operator
(MamKO)~\citep{li2025mamko}, which processes recent state and control
history through a convolutional-FCNN network to emit the matrices $\bar{A}_k$, $\bar{B}_k$, $\bar{C}_k$ of a local linear system at
each step, substantially outperforming static Koopman models while retaining
the convex MPC formulation.

\paragraph{The structural constraint.}
This class of methods rests on a deliberate architectural choice: the drift
operator $\bar{A}_k$ is generated from historical data
$\Xhist = \{z_{k-H:k-1}, u_{k-H:k-1}\}$ but is explicitly \emph{not} a function
of the current control input $u_k$. Making $\bar{A}_k$ depend on $u_k$ would
introduce a bilinear term $\bar{A}_k(u_k)\,z_k$, breaking MPC convexity.
This constraint is acknowledged in the MamKO formulation~\citep{li2025mamko}
and is forced more generally by any control architecture that preserves
single-shot QP convexity. The consequence is conditional
independence $\bar{A}_k \perp\!\!\!\perp u_k \mid \Xhist$. This constraint
limits the model's expressive power along two distinct axes.

\emph{Control--state coupling.} When the governing physics contains
$u \cdot x$ terms---convective fluxes in flow-actuated reactors,
multiplicative inputs in mechanical systems, voltage-dependent impedances in
power electronics~\citep{mohler1973bilinear}---the mixed partial
$\partial^2 z_{k+1} / \partial z_k \partial u_k^{(i)}$ is structurally
nonzero. MamKO's architecture forces this quantity to zero, requiring its
backbone to absorb the coupling indirectly via step-by-step regeneration of
$\bar{A}_k$. Over a multi-step rollout, the resulting approximation error
accumulates fastest in regimes where the controller is deliberately
exercising large control variation---precisely where MPC matters most.

\emph{Time-varying parameters.} When the dynamics themselves drift over
time---catalyst aging, heat-exchanger fouling, mechanical wear, friction
modulation---the effective operator must track changes that the historical
encoder may not yet have observed within its lookback window. A bilinear
coupling provides an additional, structured channel through which the
operator can adapt to the present operating point in a way that the linear
lift cannot reach within a single MPC horizon.

Both regimes are pervasive in industrial control: the first is structural,
the second operational. A single architectural mechanism that addresses
both---without requiring a different model class for each---is the
contribution of this paper.

\paragraph{Our contribution.}
We propose a minimal, principled extension of the Mamba-Koopman dynamics that relaxes
conditional independence while retaining the features that make the model
useful for control:
\begin{equation}
  z_{k+1} = \underbrace{\Bigl(\mathrm{diag}(A) + \sum_{i=1}^{m} u_k^{(i)}\,G_i\Bigr)}_{=:\,\Aeff(u_k)}
             z_k + B_k u_k,
  \label{eq:bilinear_dynamics}
\end{equation}
where $A \in \R^{\dz}$ contains the diagonal continuous-time eigenvalues of the existing Mamba-Koopman backbone (unchanged), $G_i \in \R^{\dz \times \dz}$ are
learned \emph{bilinear coupling tensors} and $B_k \in \R^{\dz \times m}$ is the step-varying
actuator matrix from the existing dynamics network. \Cref{eq:bilinear_dynamics}
is the model class analyzed in \Cref{prop:generalization,prop:forecasting},
written with the backbone's continuous-time eigenvalues; the implemented
discrete transition is the Lie--Trotter map of \Cref{eq:lt_dynamics}, under
which setting $G_i = 0$ restores MamKO exactly. The bilinear structure exposes exact analytical
model Jacobians, making Sequential Convex Programming~(SCP) natural: linearize
around a nominal trajectory, solve a convex QP, update, repeat.

\paragraph{Summary of contributions.}
\begin{enumerate}
  \item \textbf{Bilinear Koopman dynamics}~(\Cref{sec:method}): a strict
    generalization of the conditional-independence class modeling control--state coupling, with an optional low-rank parameterization and a spectral penalty that regularizes the stability of the resulting dense operator.
  \item \textbf{SCP controller}~(\Cref{sec:control}): an iterative MPC
    algorithm exploiting exact model Jacobians, with monotone-descent and
    KKT convergence results under standard trust-region assumptions.
  \item \textbf{Empirical evaluation}~(\Cref{sec:experiments}) across
    CartPole and RSCP in time-invariant and time-varying regimes,
    establishing four consistent effects: forecasting non-inferiority on
    every cell when training noise is averaged out, training stabilization on RSCP TV, closed-loop MPC wins on
    RSCP TV when SCP is iterated, and graceful degradation under stale-plan
    execution.
\end{enumerate}

\paragraph{Scope and limitations.}
Our method inherits the conditional-independence baseline's assumptions: the latent dynamics are
approximately linear at each step, and the encoder is expressive enough to
find a lifting where this holds. The SCP controller provides local (KKT)
rather than global optimality guarantees, and convergence rate depends on
initialization quality. We position this work as closing a specific,
well-characterized structural gap in the Koopman-for-control literature
without claiming a new operator class.

\section{Background}
\label{sec:background}

\subsection{Koopman Operators for Controlled Systems}

For a discrete-time controlled system $x_{k+1} = f(x_k, u_k)$, the Koopman
operator~\citep{koopman1931hamiltonian} acts on observables $g$ via composition:
$\mathcal{K}g = g \circ f$. This operator is \emph{linear} even when $f$ is
not. One seeks a finite-dimensional lifting $\psi: \R^n \to \R^{\dz}$ such that
the lifted state $z_k = \psi(x_k)$ evolves approximately
linearly~\citep{korda2018linear}:
\begin{equation}
  z_{k+1} = Az_k + Bu_k, \quad \hat{x}_k = Cz_k.
  \label{eq:koopman_linear}
\end{equation}
EDMD~\citep{williams2015data} identifies $A, B, C$ by least squares over a
fixed observable dictionary. Deep Koopman methods replace the fixed dictionary
with a learned encoder~\citep{lusch2018deep}.

\subsection{Mamba-Koopman and Conditional Independence}
\label{sec:mamko_background}

MamKO~\citep{li2025mamko} addresses the limitation of a fixed operator for
time-varying systems. Inspired by Mamba's selective state-space
mechanism~\citep{gu2023mamba}, a matrices-generation network takes the
historical sequence $\Xhist = [z_{k-H:k-1}^\top, u_{k-H:k-1}^\top]^\top$ and
emits time-varying operators via 1D convolution and FCNNs. After ZOH
discretization:
\begin{equation}
  z_{k+1} = \bar{A}_k(\Xhist)\,z_k + \bar{B}_k(\Xhist)\,u_k, \quad
  \hat{x}_k = \bar{C}_k z_k.
  \label{eq:mamko}
\end{equation}
Generating all matrices from $\Xhist$ alone enforces
\begin{equation}
  \bar{A}_k \perp\!\!\!\perp u_k \;\Big|\; \Xhist,
  \label{eq:cond_indep}
\end{equation}
preserving linearity in $u_k$ and hence MPC convexity, but preventing
control-state coupling.

\subsection{Bilinear Systems}

A discrete-time bilinear system takes the form
$x_{k+1} = Ax_k + \sum_{i} u_k^{(i)} N_i x_k + Bu_k$,
where $N_i \in \R^{n \times n}$ are interaction matrices. Bilinear systems
occupy a well-studied middle ground between linear and fully nonlinear: they
represent a broad class of physical phenomena including chemical reactions and
mechanical systems with multiplicative inputs~\citep{mohler1973bilinear}. The
Koopman literature has identified bilinear lifted models and related
finite-dimensional approximations as a useful middle ground for control,
making this a natural target for
Koopman approximation. SCP~\citep{mao2018successive} and iterative
LQR~\citep{li2004iterative} handle optimization over bilinear dynamics by
exploiting the analytic Jacobian structure.

\section{Bilinear Koopman Extension}
\label{sec:method}

\subsection{Bilinear Latent Dynamics}

We relax~\cref{eq:cond_indep} by parameterizing the latent transition as
\cref{eq:bilinear_dynamics}, where $A \in \R^{\dz}$, $G_i \in \R^{\dz \times
\dz}$, and $B_k$ are as described in \Cref{sec:intro}. The observation map
$\hat{x}_k = C z_k$ remains linear. The effective drift operator
$\Aeff(u_k) = \mathrm{diag}(A) + \sum_i u_k^{(i)} G_i$ is a linear function
of $u_k$, making the full dynamics bilinear in $(z_k, u_k)$ jointly.

\begin{proposition}[Strict Generalization]
\label{prop:generalization}
The function class of~\cref{eq:bilinear_dynamics} strictly contains that
of~\cref{eq:mamko}.
\end{proposition}
\begin{proof}
Setting $G_i = 0$ for all $i$ recovers \cref{eq:mamko} exactly (for the
implemented transition, \cref{eq:lt_dynamics} then reduces to the baseline
per-mode ZOH). Conversely,
any system with $\partial^2 z_{k+1}/(\partial z_k\,\partial u_k^{(i)}) =
G_i \neq 0$ is representable by~\cref{eq:bilinear_dynamics} but not
by~\cref{eq:mamko}, since \cref{eq:cond_indep} forces this cross-derivative
to zero. Hence the containment is strict.
\end{proof}

\begin{proposition}[Population-Level Forecasting Non-Inferiority]
\label{prop:forecasting}
Let $\FMamKO \subsetneq \FBL$ denote the hypothesis classes induced
by~\cref{eq:mamko} and~\cref{eq:bilinear_dynamics}. Then
\begin{multline*}
  \min_{f \in \FBL}\, \E\bigl[\|z_{k+1} - f(z_k,u_k)\|^2\bigr] \\
  \leq\; \min_{f \in \FMamKO}\, \E\bigl[\|z_{k+1} - f(z_k,u_k)\|^2\bigr],
\end{multline*}
with strict inequality whenever the true latent dynamics admit nonzero
control--state coupling.
\end{proposition}
\begin{proof}
The inequality follows from \Cref{prop:generalization}: any optimum in
$\FMamKO$ is achievable in $\FBL$ by setting $G_i = 0$. For strictness,
suppose the true latent dynamics satisfy $z_{k+1} = A^\star(u_k)\,z_k +
B^\star u_k$ with $\partial A^\star / \partial u_k^{(i)} = G_i^\star \neq 0$.
Any $f \in \FMamKO$ incurs irreducible error proportional to
$\|G_i^\star\| \cdot \|u_k - \bar{u}\|$, where $\bar{u}$ is the training-set
mean control. This error vanishes only when $u_k$ is constant across the
evaluation distribution-vacuous in any MPC setting.
\end{proof}

\begin{remark}
Empirically the strictness condition partitions the systems we evaluate.
On CartPole, where the equations of motion couple horizontal force with
pendulum geometry, the bilinear extension produces strict forecasting gains
on TI and TV. On RSCP, where heat duties enter additively and the governing
ODE contains no $u \cdot x$ terms in physical coordinates, the proposition
predicts equality unless the learned lifting introduces latent coupling
(a lift containing both $x$ and $x^2$ turns an additive input into a term
bilinear in the control and the lifted state, since for
$\dot x = f(x) + bu$, $\tfrac{d}{dt}x^2 = 2x f(x) + 2b\,xu$)---and equality is what we observe on RSCP TI, and on
RSCP TV under deployment-relevant averaging (\Cref{tab:forecast}). The
$23\%$ bilinear-vs-linear gap visible in best-checkpoint MSE on RSCP TV
reflects a training-noise asymmetry between the two models, not a violation
of the proposition: the Linear baseline oscillates throughout training while
bilinear converges smoothly, so single-epoch selection from linear's
trajectory captures one of its downward spikes. Mean-last-$50$ MSE---which averages across the spikes---restores the equality the proposition predicts.
The learned $\|G\|_F$ values are reported in \Cref{tab:gi_norms}; as
\Cref{app:gi_norms} notes, they are interpretive and do not by themselves
establish the partition.
\end{remark}

\subsection{Low-Rank Parameterization}

The tensors $\{G_i\}_{i=1}^m$ add $m \cdot \dz^2$ parameters. For typical
settings ($m \leq 5$, $\dz \leq 64$), this is at most ${\sim}20$K parameters,
under 1\% of a standard Mamba-Koopman backbone. A low-rank prior can be imposed:
\begin{equation}
  G_i = L_i R_i^\top, \quad L_i, R_i \in \R^{\dz \times r}, \quad r \ll \dz,
  \label{eq:low_rank}
\end{equation}
reducing the addition to $2mr\dz$ parameters and regularizing the coupling
strength when $r < \dz$. For $r = \dz$ (all reported runs) $G_i$ is
parameterized directly and initialized to zero; for $r < \dz$, $L_i$ is
initialized with small random values and $R_i$ at zero.
Either way $G_i = 0$ at initialization, so the model begins as an exact copy of the Linear baseline; any divergence
in training is attributable to the bilinear extension. The rank $r$ is
a single interpretable hyperparameter: $r = 1$ encodes a rank-$1$
perturbation per control channel; $r = \dz$ recovers the unconstrained
bilinear model. We use $r = \dz$ in all reported runs, as
$\dz \in \{8, 15\}$ is small enough that further parameter reduction is
unnecessary; the low-rank option remains available for deployments with
larger latent dimensions.

\subsection{Lie--Trotter Split ZOH Discretization}

Write the effective generator as
\begin{equation}
\begin{aligned}
  \Aeff(u_k) &= D + P(u_k), \\
  D &= \mathrm{diag}(A), \\
  P(u_k) &= \sum_{j=1}^{m} u_k^{(j)} G_j.
\end{aligned}
\label{eq:generator_split}
\end{equation}
Dense ZOH would exponentiate $D + P(u_k)$ at a single scalar period $T$,
collapsing the baseline MamKO per-mode timescales $\delta_n$. This would
break the exact zero-coupling reduction to the linear baseline. We therefore
use a first-order Lie--Trotter splitting~\citep{mclachlan2002splitting},
historically rooted in the Trotter product formula~\citep{trotter1959product}:
\begin{equation}
\begin{aligned}
  \exp\!\left((D + P(u_k))T\right)
  &\approx \underbrace{\exp\!\left(P(u_k)T\right)}_{=:E_P(u_k)} \\
  &\qquad\; \underbrace{\mathrm{diag}\!\left(\exp(A \odot \boldsymbol{\delta})\right)}_{=:E_D}.
\end{aligned}
\label{eq:lie_trotter}
\end{equation}
With $\bar{B}_{\mathrm{diag},k}$ the corresponding diagonal ZOH input
integral, the implemented one-step map is
\begin{equation}
  F_{\mathrm{LT}}(z_k, u_k) := E_P(u_k)\bigl(E_D z_k + \bar{B}_{\mathrm{diag},k} u_k\bigr).
\label{eq:lt_dynamics}
\end{equation}
When $P(u_k) \equiv 0$, $E_P(u_k) = I$, so the scheme reduces exactly to the
baseline per-mode diagonal ZOH. Implementation details, numerical
stabilization, and cost are deferred to \Cref{app:lie_trotter_details}.

\subsection{Stability}
\label{sec:stability}

For the diagonal map, $\rho(E_D) \leq 1$ holds whenever the
continuous-time eigenvalues satisfy $a_n \leq 0$ (with $\delta_n > 0$). The
\texttt{negative\_celu} activation in the Mamba-Koopman backbone bounds
$a_n$ above by $1$, which alone does not ensure $a_n \leq 0$. No such guarantee holds once
$G_i \neq 0$: the bilinear terms can move eigenvalues of
$A_{\mathrm{disc},k}$ outside the unit disk.
We regularize stability at training time.

\paragraph{Training: spectral penalty.}
We add to the loss
\begin{equation}
  \mathcal{L}_{\mathrm{stab}} = \sum_j \mathrm{ReLU}\bigl(|\lambda_j(A_{\mathrm{disc},k})| - 1 + m_s\bigr),
  \label{eq:stab_loss}
\end{equation}
with margin $m_s = 0.05$, computed via \texttt{torch.linalg.eigvals} in
\texttt{float32} to avoid mixed-precision artifacts. Zero initialization of
$G_i$ makes the initial spectrum that of the diagonal backbone.

\paragraph{Inference.}
No per-step spectral check is performed at inference; the spectral
penalty is applied during training only. \Cref{app:stability} discusses why
a cheap Gershgorin pre-check is uninformative here and leaves runtime
stability gating to future work.

\subsection{Training Objective}

We train end-to-end on multi-step observation-space prediction loss:
\begin{equation}
  \mathcal{L} = \frac{1}{T}\sum_{k=1}^{T} \|\hat{x}_k - x_k\|^2
              + \lambda_s\,\mathcal{L}_{\mathrm{stab}},
  \label{eq:loss}
\end{equation}
with the stability penalty disabled at inference. All baseline training hyperparameters (optimizer, schedule, context window $H$) are held fixed; only
the coupling parameters ($G_i$, or $\{L_i, R_i\}$ when factorized) and the
jointly trained backbone parameters are updated.

\section{SCP Controller for Bilinear MPC}
\label{sec:control}

\subsection{Problem Formulation}

The open-loop MPC problem over horizon $H$ is
\begin{equation}
\begin{aligned}
  \min_{u_{0:H-1}} \;\;& \sum_{k=0}^{H-1} \ell(z_k, u_k) \\
  \text{s.t.}\;\;& z_{k+1} = F_{\mathrm{LT}}(z_k, u_k), \\
                 & u_k \in \mathcal{U}.
\end{aligned}
\label{eq:mpc}
\end{equation}
This problem is non-convex in $(z_k, u_k)$ jointly. The conditional-independence baseline sidesteps
non-convexity by design (conditional independence makes the QP trivially
convex in one shot). We instead solve~\cref{eq:mpc} iteratively via SCP,
exploiting the exact differentiable structure of the implemented bilinear model.

\subsection{Exact Model Jacobians}

SCP linearizes the implemented Lie--Trotter map~\cref{eq:lt_dynamics} around a
nominal trajectory $(\bar{z}_{0:H}, \bar{u}_{0:H})$. Since
$F_{\mathrm{LT}}(z_k, u_k)$ is affine in $z_k$, the state Jacobian is
\begin{align}
  \Atil_k
  &= \frac{\partial F_{\mathrm{LT}}}{\partial z_k}\Bigg|_{(\bar{z}_k,\bar{u}_k)}
   = E_P(\bar{u}_k) E_D,
  \label{eq:A_tilde} \\
  \Btil_k
  &= \frac{\partial F_{\mathrm{LT}}}{\partial u_k}\Bigg|_{(\bar{z}_k,\bar{u}_k)}.
  \label{eq:B_tilde}
\end{align}
The input Jacobian $\Btil_k$ contains the derivative of the matrix
exponential $E_P(u_k)$ with respect to $u_k$. In implementation, we compute
$\Btil_k$ by reverse-mode automatic differentiation of the full
Lie--Trotter update using \texttt{torch.func.jacrev}. This yields exact model
Jacobians of the implemented differentiable dynamics up to floating-point
precision, with no finite-difference linearization inside SCP.

\subsection{SCP Algorithm}

Each SCP iteration solves the convex subproblem obtained by evaluating the
finite-horizon quadratic tracking objective on the linearized trajectory.
Denoting this objective by $J_{\mathrm{lin}}$, we solve
\begin{equation}
\begin{aligned}
  \min_{\delta u_{0:H-1}} \;\;& J_{\mathrm{lin}}(\delta u; \bar{z}, \bar{u}) \\
  \text{s.t.}\;\;& \delta z_0 = 0, \\
                 & \delta z_{k+1} = \Atil_k \delta z_k + \Btil_k \delta u_k, \\
                 & \bar{u}_k + \delta u_k \in \mathcal{U}, \\
                 & \|\delta u_k\|_\infty \leq \varepsilon.
\end{aligned}
\label{eq:qp}
\end{equation}
Here $J_{\mathrm{lin}}$ includes the same stage and terminal terms as the true
MPC objective; when a $\Delta u$ penalty is used, the previous applied input is
treated as fixed data, so the subproblem remains a convex QP in
$\delta u_{0:H-1}$. Eliminating the linearized dynamics yields a dense
box-constrained QP solved via OSQP~\citep{stellato2020osqp} with warm-starting
between iterations.

\begin{algorithm}[t]
\caption{SCP-Bilinear MPC}
\label{alg:scp}
\begin{algorithmic}[1]
\Require Initial state $z_0$, nominal controls $\bar{u}_{0:H-1}$, initial trust-region radius $\varepsilon_0$, model params
\Ensure Optimized control sequence $u_{0:H-1}^\star$
\State $\varepsilon \leftarrow \varepsilon_0$
\State Roll out $\bar{z}_{0:H}$ under \cref{eq:lt_dynamics}
\For{$n = 1, \ldots, N_{\mathrm{SCP}}$}
  \State Compute $\Atil_k$, $\Btil_k$ via \crefrange{eq:A_tilde}{eq:B_tilde}
  \State Eliminate $\delta z$; form dense QP in $\delta u_{0:H-1}$
  \State Solve QP via OSQP (warm-started)
  \If{$J(\bar{u} + \delta u) \leq J(\bar{u})$}
    \State $\bar{u} \leftarrow \bar{u} + \delta u$; roll out $\bar{z}_{0:H}$
  \Else
    \State $\varepsilon \leftarrow \varepsilon / 2$
  \EndIf
\EndFor
\State \Return $\bar{u}_{0:H-1}$
\end{algorithmic}
\end{algorithm}

\begin{proposition}[SCP Monotone Descent and KKT Convergence]
\label{prop:scp_convergence}
Consider the bilinear MPC problem~\cref{eq:mpc} with a continuously
differentiable finite-horizon objective $J$ whose stage and terminal terms are
convex quadratic in the variables of each SCP subproblem (including, when
used, control-increment penalties with the previous input treated as fixed
problem data), and with box constraints
$\mathcal{U} = \{u : u_{\min} \leq u \leq u_{\max}\}$. Suppose the SCP
iterates $\{(\bar{u}^{(n)}, \bar{z}^{(n)})\}$ remain in a bounded set, the
implemented dynamics map $F_{\mathrm{LT}}$ of~\cref{eq:lt_dynamics} has locally
Lipschitz Jacobian on a neighborhood of that set, and each convex subproblem
includes both the feasibility constraint $\bar{u}_k + \delta u_k \in
\mathcal{U}$ and the trust-region constraint $\|\delta u_k\|_\infty \leq
\varepsilon$. Then:

\vspace{2pt}
\noindent\textbf{(i) Monotone descent.}
Let $\Delta^{(n)} > 0$ denote the predicted QP cost decrease at iteration $n$.
Because $F_{\mathrm{LT}}$ has locally Lipschitz Jacobian, the first-order model
error on a trust region of radius $\varepsilon$ is $O(\varepsilon^2)$
uniformly on bounded sets. Hence the discrepancy between the linearized and
true objectives over a finite horizon satisfies
\begin{equation}
  |J_{\mathrm{lin}} - J_{\mathrm{true}}| = O(\varepsilon^2).
  \label{eq:cost_discrepancy}
\end{equation}
When $\varepsilon$ is chosen so that this approximation error is less than
$\Delta^{(n)}$, the accepted step is truly descent, so the sequence
$\{J(\bar{u}^{(n)})\}$ is non-increasing. The trust-region shrinkage
$\varepsilon \leftarrow \varepsilon/2$ on failed steps guarantees this
condition is eventually met because the right-hand side of
\cref{eq:cost_discrepancy} vanishes as $\varepsilon \to 0$.

\vspace{2pt}
\noindent\textbf{(ii) KKT convergence.}
If the standard regularity assumptions for trust-region SCP hold at an
accumulation point---in particular a constraint qualification for the active
box constraints---then any accumulation point $(\bar{u}^\star, \bar{z}^\star)$
of the accepted iterates is a KKT point of~\cref{eq:mpc}. Equivalently, the
predicted QP decrease vanishes only at first-order stationary points of the
original nonlinear MPC problem, and the associated multipliers satisfy the KKT
conditions by the standard trust-region SCP argument of
\citet[Theorem~2]{mao2018successive}.
\end{proposition}

\begin{proof}
By construction, the matrices $\Atil_k, \Btil_k$ in
\crefrange{eq:A_tilde}{eq:B_tilde} are the exact Jacobians of the
implemented Lie--Trotter map $F_{\mathrm{LT}}$, computed without finite
differences.

For part (i), local Lipschitz continuity of the Jacobian of $F_{\mathrm{LT}}$
yields a quadratic first-order remainder on each trust region. Over a finite
horizon, boundedness of the iterates implies a uniform constant in the bound,
so $|J_{\mathrm{lin}} - J_{\mathrm{true}}| = O(\varepsilon^2)$ on the trust
region. Hence sufficiently small trust regions guarantee that the true cost
decrease matches the predicted decrease up to higher-order error, which yields
monotone descent after finitely many shrinkage steps.

For part (ii), feasibility is preserved because each QP enforces
$\bar{u}_k + \delta u_k \in \mathcal{U}$. Once the predicted decrease tends to
zero, the accepted iterates satisfy the first-order necessary conditions of the
nonlinear MPC problem under the standard trust-region SCP regularity
assumptions. The cited result of~\citet[Theorem~2]{mao2018successive} then
implies that every accumulation point of the accepted iterates is a KKT point
of~\cref{eq:mpc}.
\end{proof}

\begin{remark}
\label{rem:boundedness}
The boundedness assumption on iterates is retained as an explicit hypothesis.
In practice it is promoted by compact control constraints together with the
spectral regularization of \Cref{sec:stability},
but those mechanisms do not by themselves constitute a formal proof that
the discrete transition satisfies $\rho(A_{\mathrm{disc}}(u)) < 1$ uniformly for all $u \in \mathcal{U}$.
\end{remark}

\begin{corollary}
The conditional-independence baseline~\cref{eq:cond_indep} solves a single
QP per MPC step but cannot represent control--state coupling. The proposed
method solves $N_{\mathrm{SCP}} \in \{1, 5\}$ QPs on the implemented bilinear
model. Both are polynomial-time per iteration; with $N_{\mathrm{SCP}} = 5$ each
control step is roughly $10\times$ slower than the single-QP baseline
(\Cref{tab:compute}); this fits within the $18$\,s RSCP sampling period,
while neither controller meets CartPole's $20$\,ms period in our
implementation.
\end{corollary}

\section{Experiments}
\label{sec:experiments}

\subsection{Setup}
\label{sec:setup}

\paragraph{Systems.}
We evaluate on two benchmarks across time-invariant and time-varying regimes,
yielding four cells.

\textbf{CartPole.} Standard underactuated cart-pole, upright stabilization; $4$ states
$(x, \dot{x}, \theta, \dot{\theta})$, $1$ control (horizontal force). The
horizontal force enters the equations of motion through products with
trigonometric functions of $\theta$, making the system control-affine with
state-dependent input gain $g(x)$. The Koopman generator of such a system
acts bilinearly on observables in the control
input~\citep{goswami2020global}, so any sufficiently
expressive lifting that captures $\cos\theta$, $\sin\theta$, and
$\dot{\theta}^2$ inherits a $u \cdot z$ coupling that MamKO's
conditional-independence constraint excludes. Whether the learned encoder
spans these coordinates well enough for the gap to be empirically
detectable is one of the questions our forecasting experiments answer
(\Cref{tab:forecast}). The time-invariant variant (\textbf{TI}) uses
no cart or pole friction (frictionless limit, \Cref{app:cartpole}); the time-varying variant (\textbf{TV}) modulates
the cart friction term as
$\mu_c(t) = \mu_c^{\mathrm{base}} + \sin(\omega t)$, a signed synthetic modifier (\Cref{app:cartpole}).

\textbf{RSCP.} The reactor--separator process of~\citet{liu2008distributed},
adopted as the MamKO benchmark in~\citet{li2025mamko}. Two CSTRs in series
followed by a flash separator with recycle; $9$ states and $3$ controls (heat
duties $Q_1, Q_2, Q_3$). The heat duties enter the energy balances
\emph{additively} as $+Q_i / (\rho c_p V_i)$: the governing ODE contains no
$u \cdot x$ terms, only nonlinearity in state. The TV variant multiplies both
Arrhenius rate terms by a catalyst-deactivation factor
$\varphi_c(t) = e^{-0.01 t}$ (\Cref{app:rscp}), introducing time-varying
parameters without adding control--state coupling in physical coordinates. RSCP exercises the second axis of \Cref{sec:intro}
(time-varying parameters) without the first (control coupling), making the
two benchmarks complementary tests of the bilinear extension.

\paragraph{Baselines.}
Our primary structural ablation is the conditional-independence baseline,
which is equivalent to our model with $G_i = 0$. We instantiate this
baseline using the MamKO architecture~\citep{li2025mamko}, which is
representative of the class. Throughout, ``Linear'' refers to this
$G_i = 0$ baseline. MamKO itself was evaluated against the Deep Koopman
Operator (DKO) in the original work and established as the stronger
time-varying Koopman baseline; we inherit this comparison by transitivity.
Our method is denoted \textbf{Bilinear-$r$}; we use $r = d_z$ (full rank)
in all reported runs, giving $r = 8$ on CartPole and $r = 15$ on RSCP
(\Cref{tab:arch}).
Closed-loop variants are \textbf{Bilinear-SCP-$N$} for $N \in \{1, 5\}$
SCP iterations.

\paragraph{Protocol.}
We generate data with the MamKO \texttt{ReplayMemory} protocol
(\Cref{app:data}): overlapping $60$-step windows from rollouts under per-step
i.i.d.\ uniform control excitation, pooled to $39{,}900$ training/validation
windows ($80/20$ split) and $4{,}000$ test windows per system variant. All Koopman-based methods use the same latent dimension $\dz$,
prediction horizon, and training budget; dataset splits and random seeds are
held constant. Forecasting is reported as 30-step open-loop MSE on held-out
test windows. Closed-loop MPC is evaluated over $10$ episodes per cell
under a setpoint-tracking task; we report cumulative running-average cost as
a function of time, with shaded $0.3\sigma$
bands\footnote{Bands are reported at $0.3\sigma$ for visual readability of
model separation; they show variability across the $10$ episodes, not
confidence intervals. At $1\sigma$ some panels exhibit fully overlapping
bands. We report the band width
explicitly so readers can interpret the figures under the variance scale of
their preference.} for visual comparison.

\paragraph{Reproducibility against published values.}
Our re-trained Linear baseline reproduces the forecast MSE
of~\citet[Table~1]{li2025mamko} within $\sim 3\%$ on CartPole TI and within
$0.78\times$ on RSCP TI. On CartPole TV our Linear baseline trains to MSE
$7.8 \times 10^{-3}$ versus the published $4.3 \times 10^{-4}$, a gap we
have not closed under matched hyperparameters; on RSCP TV our Linear
baseline trains to $7.1 \times 10^{-4}$ versus the published
$9.7 \times 10^{-3}$, a $13\times$ improvement that in part reflects
best-checkpoint selection on Linear's oscillating validation curve
(\Cref{sec:forecast,sec:training_stability}).
Bilinear-vs-Linear comparisons throughout this paper are reported against
our own re-trained baseline; reviewers comparing absolute numbers should
reference this calibration.

\subsection{Time-Varying Capture: SCP Iteration and Lead-Time Robustness}
\label{sec:tv_capture}

The headline cell of this paper is \textbf{RSCP TV}: a chemical reactor
benchmark with exponentially decaying reaction rates (catalyst deactivation), no $u \cdot x$
coupling in its governing ODE, and time-varying parameters that the linear
lift can only track through history-driven re-generation of the operator at
each MPC step. Two experiments converge on a single finding: under TV
physics, the bilinear coupling provides operator-level correction that the
linear lift architecturally cannot reach within a single MPC horizon, but
unlocking it requires either iterative re-linearization at solve time or
commitment to a control sequence whose effect on the operator is
structurally encoded.

\paragraph{SCP iteration unlocks TV capture under standard MPC.}
\Cref{fig:mpc_eval} reports cumulative running-average closed-loop cost for
Linear (single QP), Bilinear-SCP-$1$ (one SCP iteration on the bilinear
model---equivalent to LQR-style single linearization), and Bilinear-SCP-$5$
(five iterations) across the four cells. On RSCP TV (bottom right),
Bilinear-SCP-$5$ separates from both other controllers at $t \approx 0.4$\,h
and the gap widens monotonically through end of horizon. Bilinear-SCP-$1$,
by contrast, does not distinguish itself from Linear and trails it slightly.
A single linearization of the bilinear model is insufficient to exploit the
additional capacity under TV physics: the SCP iterations are what unlock
the operator-level corrections the bilinear coupling makes available, by
repeatedly re-linearizing
$\Aeff(u_k) = \mathrm{diag}(A) + \sum_i u_k^{(i)} G_i$ at the current
operating point as the iterates proceed.

\Cref{tab:mpc_cost} reports per-cell closed-loop tracking cost for Linear
and our Bilinear-SCP-5 controller over the matched evaluation window
of~\citet[Fig.~3]{li2025mamko}. Bilinear-SCP-5 wins on three of four cells,
with the largest gap on RSCP TV ($3.44$ vs $4.91 \times 10^{-3}$, a $30\%$
reduction). The single loss is a $6\%$ gap on CartPole TI; this regime
has no time-varying parameters and benign
dynamics, so the additional SCP overhead does not compound into a benefit.

\begin{table*}[t]
  \centering
  \caption{Closed-loop MPC cumulative cost over 10 episodes per cell, on
    the matched evaluation window of~\citet[Fig.~3]{li2025mamko}.
    \textbf{Bold}: best per cell.}
  \label{tab:mpc_cost}
  \begin{tabular}{lrr}
    \toprule
    \textbf{Cell} & \textbf{Linear} & \textbf{Bilinear-SCP-5} \\
    \midrule
    CartPole TI (20\,s)   & $\mathbf{2.40{\times}10^{-2}}$ & $2.54{\times}10^{-2}$ \\
    CartPole TV (20\,s)   & $1.79{\times}10^{-2}$          & $\mathbf{1.73{\times}10^{-2}}$ \\
    RSCP TI (2\,h)        & $7.90{\times}10^{-4}$          & $\mathbf{6.43{\times}10^{-4}}$ \\
    RSCP TV (2\,h)        & $4.91{\times}10^{-3}$          & $\mathbf{3.44{\times}10^{-3}}$ \\
    \bottomrule
  \end{tabular}
\end{table*}

We emphasize that Bilinear-SCP-$1$ is not a stripped-down ablation but a
realistic deployment baseline: it is what one obtains from running the
bilinear model under a budget-constrained MPC stack (e.g., compute-limited
edge controllers). The separation between SCP-$1$ and SCP-$5$ on RSCP TV
is the empirical cost of single linearization in the TV regime.

\begin{figure}[ht]
  \centering
  \includegraphics[width=\columnwidth]{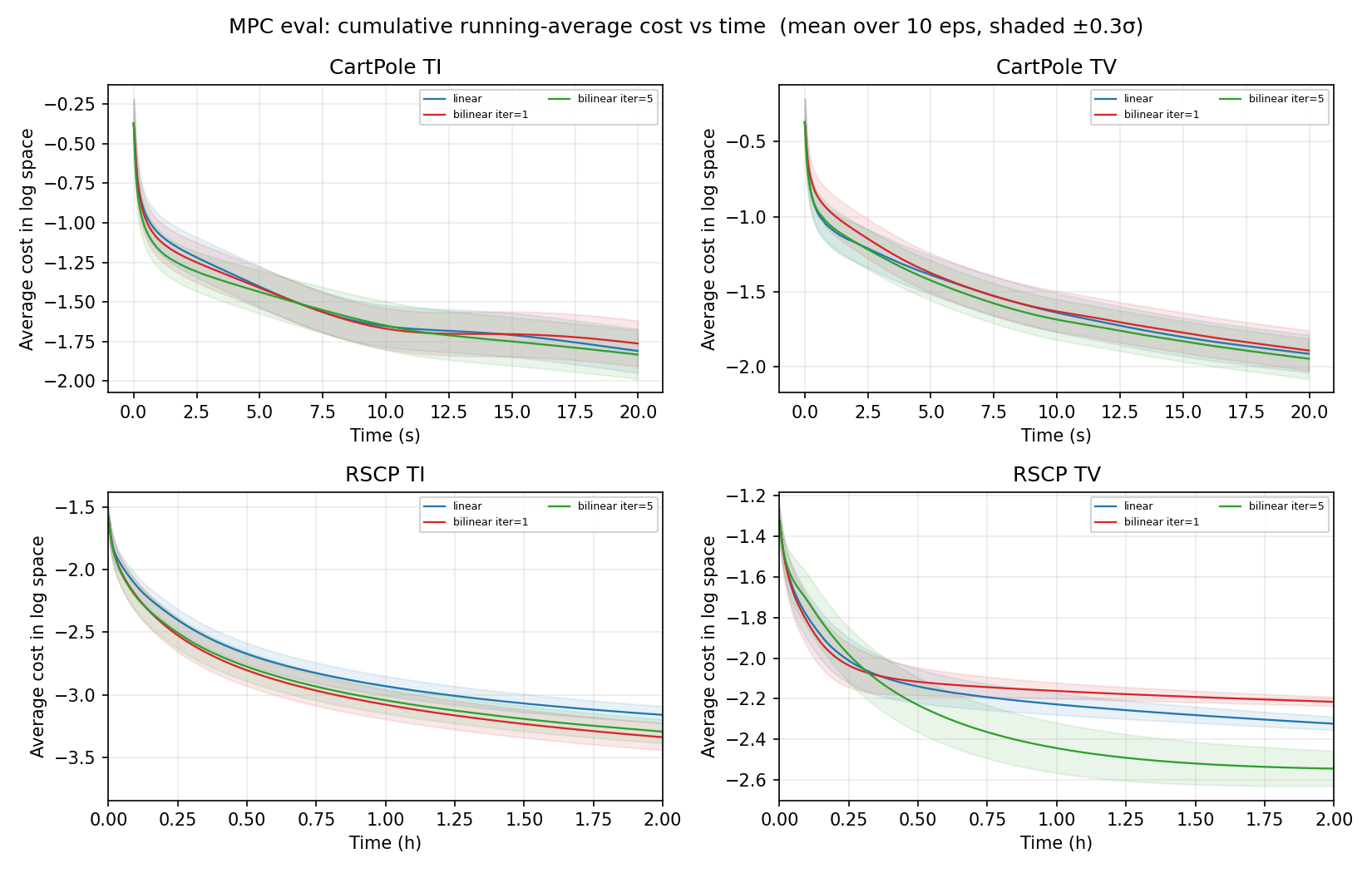}
  \caption{Closed-loop MPC cost over $10$ episodes per cell.
    Bilinear-SCP-$5$ separates from both other controllers on RSCP TV
    (bottom right) starting around $t \approx 0.4$\,h, with the gap widening
    monotonically through end of horizon. Other cells discussed in
    \Cref{sec:other_cells}.}
  \label{fig:mpc_eval}
\end{figure}

\paragraph{Lead-time commitment unlocks the same correction with a frozen backbone.}
In deployment, MPC controllers are not always free to re-plan at every
control step: compute budgets, network round-trip delays, and asynchronous
sensor updates force the controller to commit to a previously computed
plan for some lead time before re-planning. We probe this regime via a
\emph{lead-time experiment}: at each MPC step the controller commits to
its plan for $d$ subsequent steps before re-planning, with
$d \in \{0, 1, 3, 5\}$. The $d = 0$ case recovers standard receding-horizon
MPC. Crucially, we hold \emph{both} the original plan and the backbone
dynamics fixed during the lead window (\texttt{regen=never}),
so neither model receives intermediate corrections from new sensor data.
This isolates the open-loop correction capacity of the model class itself:
the linear lift's drift $\bar{A}_k$ is fixed across the entire lead window
because it depends only on history; the bilinear model's effective drift
$\Aeff(u_k)$ varies with the \emph{committed} control $u_k$ along the
lead window, providing operator-level correction even with the backbone
frozen. We exclude $d = 10$ from analysis: at over three minutes of stale
plan on RSCP, both models saturate from discretization-induced error
rather than model-class differences.

\Cref{fig:lead_time_rscp} shows the result on RSCP TV. The bilinear model
sits below Linear at every commitment window
(\Cref{tab:lead_time_full}), with the gap widening sharply at
$d \in \{1, 3\}$. At $d = 0$ the bilinear advantage is modest ($-2.54$ vs
$-2.32$, a gap of $0.22$ in log-cost). At $d = 1$, Linear degrades
sharply---its cumulative log-cost plateaus near $-1.87$ by end of horizon,
while the bilinear model reaches $-2.78$, a gap of nearly $1$ in log-cost.
The pattern persists at $d = 3$ (Linear plateau $-2.06$, bilinear $-2.57$)
and at $d = 5$ bilinear still leads by $\sim 0.22$ in log-cost.
\Cref{fig:lead_time_cartpole} reports the same experiment on CartPole TV.
The bilinear advantage is small, and bilinear stays below Linear at every $d$
(\Cref{tab:lead_time_cartpole_full}); the underlying dynamics are benign
enough that plan staleness does not catastrophically degrade either model.

\begin{figure}[ht]
  \centering
  \includegraphics[width=\columnwidth]{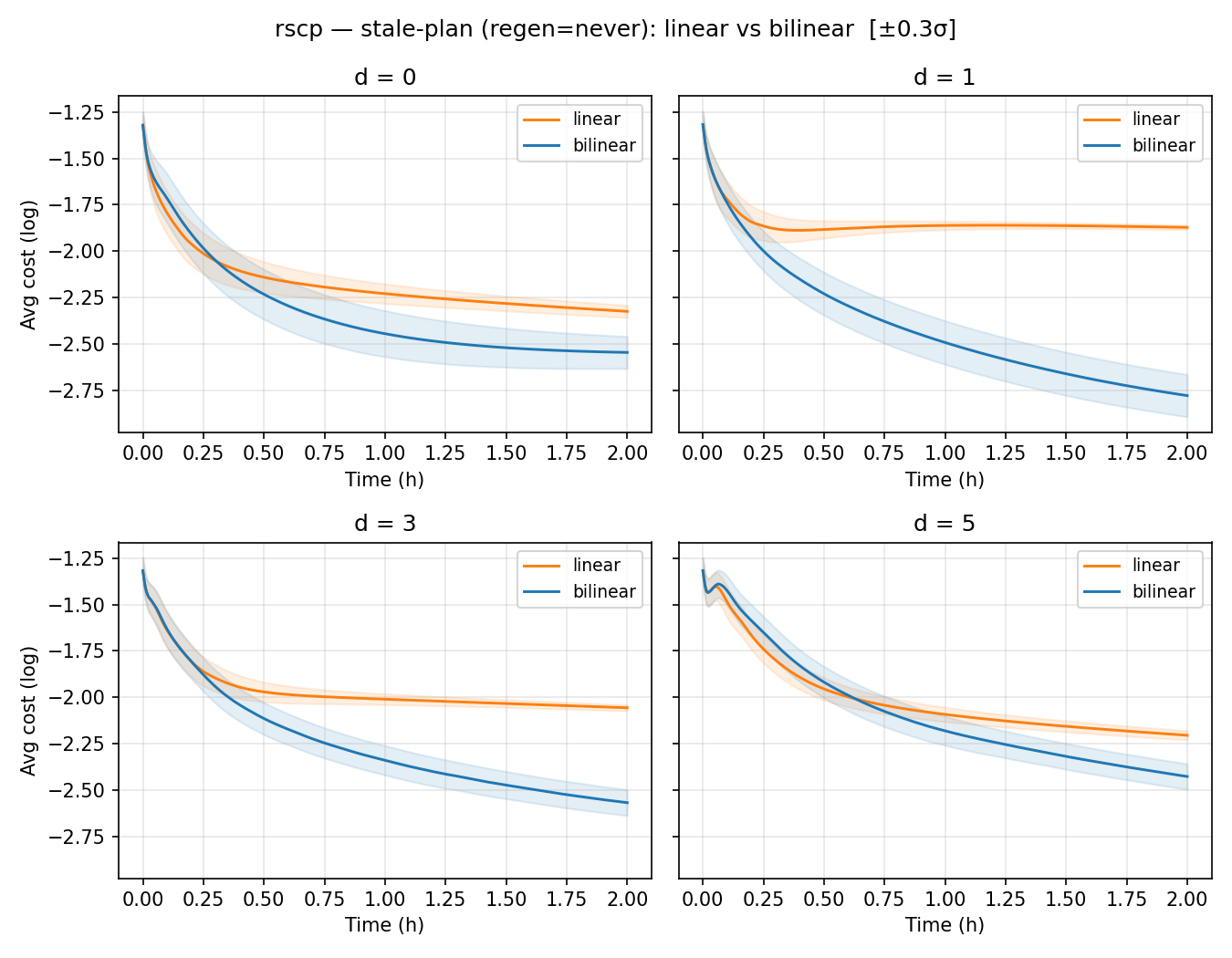}
\caption{Lead-time robustness on RSCP TV under stale-plan execution
    (\texttt{regen=never}): both the plan and the backbone are frozen
    during the lead window. The bilinear model maintains an advantage at
    every $d$, with the gap widening sharply at $d \in \{1, 3\}$ where
    Linear degrades while bilinear continues to converge; at $d = 5$
    bilinear still leads. Bands are $\pm 0.3 \sigma$
    over $10$ episodes.}
  \label{fig:lead_time_rscp}
\end{figure}

\begin{figure}[ht]
  \centering
  \includegraphics[width=\columnwidth]{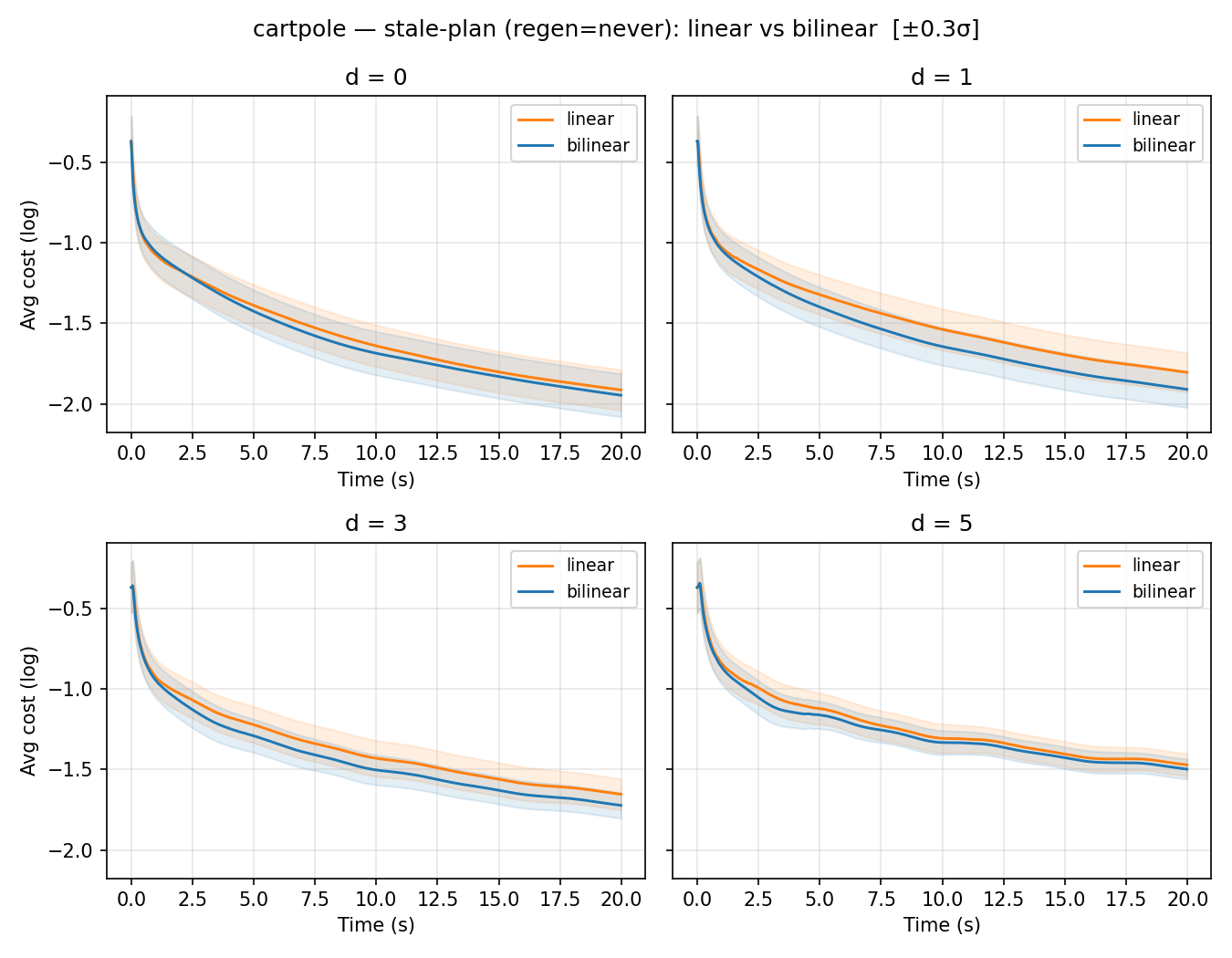}
\caption{Lead-time robustness on CartPole TV under stale-plan execution.
    Bilinear sits modestly below Linear at every $d$. The system is benign enough that staleness
    does not catastrophically degrade either model. Bands are
    $\pm 0.3 \sigma$ over $10$ episodes.}
  \label{fig:lead_time_cartpole}
\end{figure}

\paragraph{One mechanism, two regimes.}
Both findings identify the same mechanism. The bilinear coupling encodes
how the drift varies with the control, and this variation is exploitable
through any control regime that exercises the dependency: iterated SCP
exercises it via re-linearization at the current operating point;
lead-time commitment exercises it via rollout of the committed control
sequence with the backbone frozen. The Linear baseline, lacking the
coupling, has no analogous mechanism: its only access to time-varying
physics is backbone re-generation, which fails under freeze and is
insufficient under standard MPC. This regime is operationally relevant:
control deployments commonly enforce re-planning at intervals longer
than the underlying simulator step due to compute, communication, or
scheduling constraints.

\subsection{Forecasting Non-Inferiority}
\label{sec:forecast}

\begin{table*}[t]
  \centering
  \caption{Forecast MSE on held-out test windows. Best-checkpoint
    (single minimum-validation-loss epoch, the convention
    of~\citet[Table~1]{li2025mamko}) and mean over the final $50$ training
    epochs (deployment-relevant, robust to selection from noisy training
    curves). \textbf{Bold}: best per cell per metric.}
  \label{tab:forecast}
  \begin{tabular}{l l rrrr}
    \toprule
    \textbf{Model} & \textbf{Metric} & \textbf{CP TI} & \textbf{CP TV}
                                     & \textbf{RSCP TI} & \textbf{RSCP TV} \\
    \midrule
    \multirow{2}{*}{Linear}
      & best        & $6.31{\times}10^{-4}$ & $7.80{\times}10^{-3}$
                    & $\mathbf{2.28{\times}10^{-3}}$
                    & $\mathbf{7.14{\times}10^{-4}}$ \\
      & mean$_{50}$ & $6.67{\times}10^{-4}$ & $8.01{\times}10^{-3}$
                    & $2.36{\times}10^{-3}$
                    & $9.09{\times}10^{-4}$ \\
    \midrule
    \multirow{2}{*}{Bilinear}
      & best        & $\mathbf{4.45{\times}10^{-4}}$ & $\mathbf{7.43{\times}10^{-3}}$
                    & $2.28{\times}10^{-3}$ & $8.80{\times}10^{-4}$ \\
      & mean$_{50}$ & $\mathbf{4.77{\times}10^{-4}}$ & $\mathbf{7.77{\times}10^{-3}}$
                    & $\mathbf{2.36{\times}10^{-3}}$ & $\mathbf{9.05{\times}10^{-4}}$ \\
    \bottomrule
  \end{tabular}
\end{table*}

\Cref{tab:forecast} reports forecast MSE under best-checkpoint (the single
epoch with lowest validation loss, matching the convention
of~\citet[Table~1]{li2025mamko}) and mean over the final $50$ training
epochs (the value a deployed system would see, robust to selection from
noisy training curves). The two metrics agree on three of four cells:
bilinear wins CartPole TI by $\sim 29\%$, edges CartPole TV by $3$--$5\%$,
and ties RSCP TI within noise---all consistent with
\Cref{prop:forecasting}.

The cells diverge on RSCP TV. Best-checkpoint favors Linear by $23\%$;
mean$_{50}$ ties (bilinear marginally lower at $9.05$ vs
$9.09 \times 10^{-4}$). \Cref{fig:val_loss} (bottom right) explains the
gap: the Linear baseline trains into a noisy basin that oscillates with
amplitude $\pm 0.5$ in log-loss for the entire $400$-epoch run, while
bilinear converges smoothly within $\sim 50$ epochs and remains flat.
Best-checkpoint selects one of Linear's downward spikes; the tail mean
averages across them. On an oscillating curve, best-checkpoint
selection is a noisy estimate of typical performance; the tail mean is the
more robust comparison.

\subsection{Training Stability}
\label{sec:training_stability}

Validation-loss curves (\Cref{fig:val_loss}) reveal a striking asymmetry
between cells. On CartPole TI/TV and RSCP TI, both models train smoothly
to convergence. On RSCP TV, the Linear model exhibits sustained
oscillations of $\pm 0.5$ in log-loss throughout the entire $400$-epoch
run, while the bilinear model converges smoothly within the first
$\sim 50$ epochs and remains flat thereafter. The variance of the Linear
model's val loss across the final $200$ epochs exceeds the bilinear
model's by an order of magnitude.

We interpret this as an optimization-side benefit of the bilinear
parameterization: under time-varying physics, the linear lift must
continuously adjust its operator generation to track the moving target,
producing the noisy training signal we observe. The bilinear coupling
absorbs a portion of the time-varying capacity directly into a stationary
operator structure, decoupling the optimization landscape from the data's
non-stationarity. The headline best-checkpoint number for RSCP TV
(bilinear $+23\%$ worse than Linear) and the tied mean$_{50}$ comparison
(\Cref{tab:forecast}) together tell a single story: Linear's apparent
advantage is the selection artifact, while bilinear achieves comparable
forecasting performance with dramatically more reliable training---a
property that matters for deployment even when raw MSE does not separate
the models.

The structural origin of this stabilization---the bilinear coupling
absorbing TV capacity into a stationary operator---is the same mechanism
that produces the closed-loop gains documented in \Cref{sec:tv_capture}.
The training-time and inference-time effects of the coupling are two
faces of one phenomenon.

\begin{figure}[ht]
  \centering
  \includegraphics[width=\columnwidth]{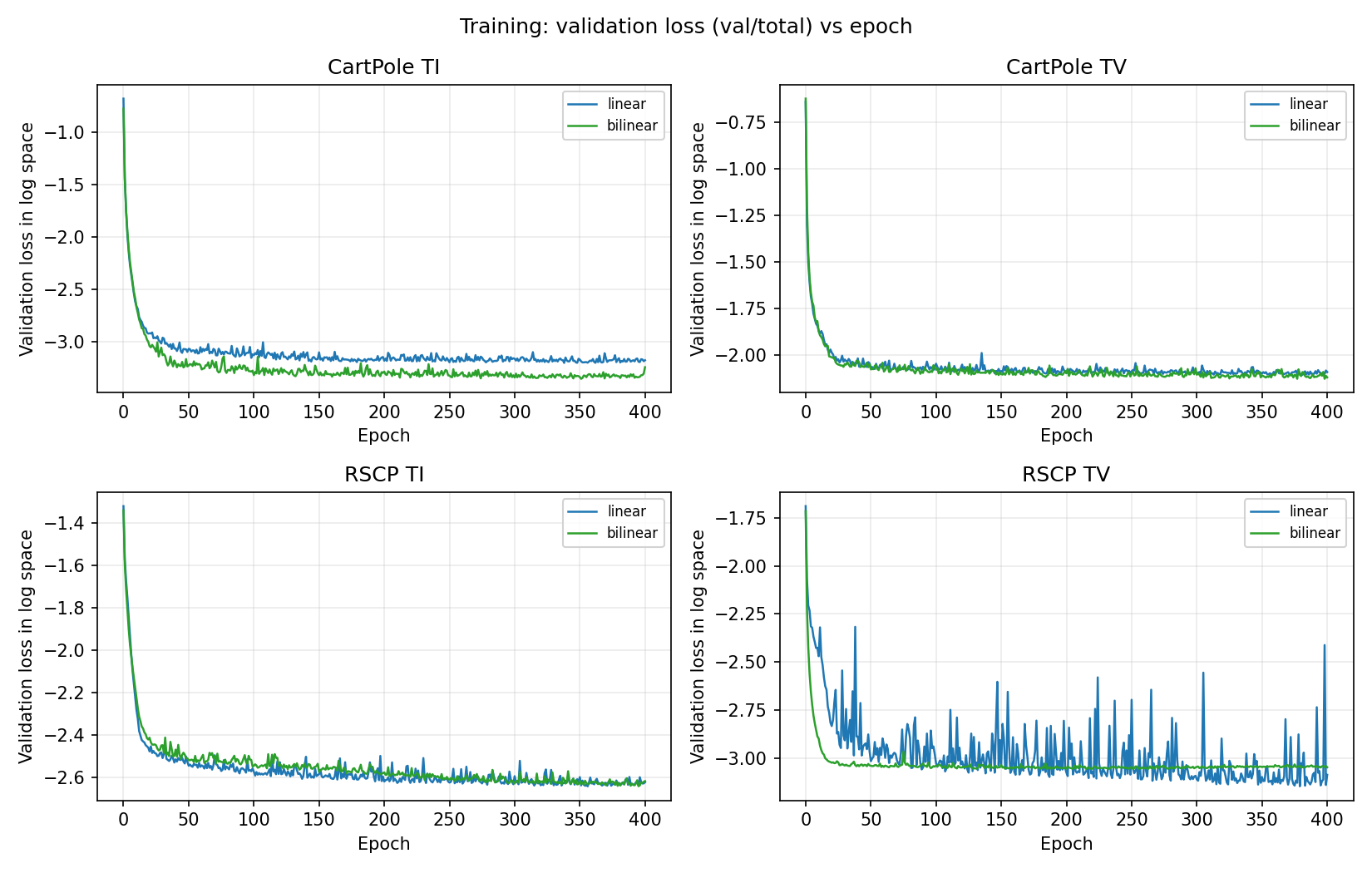}
  \caption{Validation loss over training. RSCP TV (bottom right) exhibits
    sustained oscillation of the Linear model's val loss; the bilinear
    model converges smoothly. The variance gap across the final $200$
    epochs is an order of magnitude.}
  \label{fig:val_loss}
\end{figure}

\subsection{Other Cells: CartPole and RSCP TI}
\label{sec:other_cells}

The remaining cells (CartPole TI/TV, RSCP TI) exercise the bilinear
extension under conditions in which it is either underutilized (CartPole,
where forecasting wins do not translate to closed-loop separation) or
operating outside its primary regime (RSCP TI, no time-varying parameters,
no physical $u \cdot x$ coupling). They are reported here for completeness; none
is the headline cell.

\textbf{CartPole TI/TV.} All three closed-loop controllers (Linear,
Bilinear-SCP-$1$, Bilinear-SCP-$5$) track within overlapping bands on both
variants (\Cref{fig:mpc_eval}, top row); the bilinear extension does not
hurt closed-loop performance even when its strict forecasting advantage
(\Cref{tab:forecast}) does not translate into MPC gains. Lead-time
results (\Cref{fig:lead_time_cartpole}) show bilinear sitting modestly
below Linear at every $d$. The system is benign enough that staleness does not
catastrophically degrade either model.

\textbf{RSCP TI.} Both bilinear variants sit below Linear from
$t \approx 0.5$\,h through the end of horizon, with $0.3\sigma$ bands
clearly separated by $t \approx 1.5$\,h (\Cref{fig:mpc_eval}, lower left).
Bilinear-SCP-$1$ marginally edges Bilinear-SCP-$5$ here---unlike the
RSCP TV pattern, single linearization is sufficient on time-invariant
dynamics. Forecast MSE is tied on this cell (\Cref{tab:forecast}), so we attribute
the bilinear gain to the additional lifting capacity offered by $G$, which
is useful when the controller exercises the input range broadly during
transients, even without TV physics or physical $u \cdot x$ terms.

\subsection{Coupling Strength Across Cells}
\label{sec:gi_norms_summary}

The learned $\|G\|_F$ values vary substantially across cells, from
$0.36$ on CartPole TI to $0.76$ on RSCP TI. They are interpretive and
are not, on their own, evidence of physical coupling. Per-cell values and qualitative training-trajectory
shapes are reported in \Cref{app:gi_norms}.

\subsection{Computational Cost}
\label{sec:compute}

For the TV stale-plan experiments, \Cref{tab:compute} reports mean
wall-clock per control step under lead-time execution for the Linear
controller and Bilinear-SCP-$5$ across
commitment windows $d \in \{0,1,3,5\}$. Bilinear-SCP-$5$ is substantially
slower than the single-QP Linear baseline but remains practical relative to
the $18$\,s RSCP sampling period; on CartPole, the timings should be read as
an offline stress-test rather than a real-time deployment target. The
decrease on CartPole TV with larger commitment window appears to be
problem-dependent: the fixed prefix provides a better SCP warm start and
often triggers earlier termination, whereas on RSCP TV the solver typically
still reaches the maximum SCP iteration count.

\begin{table}[!t]
  \centering
  \caption{Mean wall-clock per control step under lead-time execution on the time-varying benchmarks}
  \label{tab:compute}
  \setlength{\tabcolsep}{3pt}
  \footnotesize
  \begin{tabular}{lrrrr}
    \toprule
    \textbf{Method} & $d=0$ & $d=1$ & $d=3$ & $d=5$ \\
    \midrule
    CP --- Linear        & $0.077$\,s & $0.077$\,s & $0.078$\,s & $0.078$\,s \\
    CP --- Bilinear & $0.844$\,s & $0.818$\,s & $0.774$\,s & $0.770$\,s \\
    RSCP --- Linear      & $0.099$\,s & $0.095$\,s & $0.098$\,s & $0.098$\,s \\
    RSCP --- Bilinear & $0.876$\,s & $0.875$\,s & $0.889$\,s & $0.905$\,s \\
    \bottomrule
  \end{tabular}
\end{table}
\section{Related Work}
\label{sec:related}

\paragraph{Koopman operators for control.}
\citet{korda2018linear} first demonstrated that EDMD-based Koopman models enable
efficient convex MPC. Subsequent work has extended this to probabilistic
models~\citep{han2021desko}, graph neural network
liftings~\citep{li2020compositional}, and time-varying-system
modeling~\citep{hao2022deep}. MamKO~\citep{li2025mamko} introduced
time-varying operators via Mamba-inspired matrix generation. Our work is
orthogonal: we introduce a structural extension (bilinear coupling) that any of
these could in principle adopt.

\paragraph{Bilinear dynamical systems.}
The theory of bilinear systems is classical~\citep{mohler1973bilinear} and their
controllability properties well characterized~\citep{Elliott2009bilinear}. The
contribution here is bringing this structure into the learned Koopman-for-control
pipeline, with an optional low-rank parameterization
and a tractable SCP controller that exploits the resulting Jacobian structure.

\paragraph{Sequential convex programming.}
SCP methods have a long history in aerospace trajectory
optimization~\citep{mao2018successive,malyuta2022convex}. Our application is
standard; the novelty is that the bilinear Koopman model provides exact
model Jacobians~\crefrange{eq:A_tilde}{eq:B_tilde}, avoiding the
finite-difference approximations that degrade SCP convergence in black-box
settings.

\paragraph{Data-driven nonlinear control.}
SINDy-C~\citep{brunton2016sindy,kaiser2018sparse} can learn bilinear terms but
requires a fixed polynomial library and does not scale to high-dimensional
systems. Neural ODE methods~\citep{chen2018neural} are highly expressive but
relinquish the convex MPC structure. Our method occupies a deliberate middle
ground: more expressive than linear Koopman, more tractable than black-box
dynamics.

\section{Conclusion}
\label{sec:conclusion}

We have shown that the conditional-independence constraint shared across
the Mamba-Koopman family, while necessary to preserve MPC convexity in
single-shot formulations, leaves two regimes structurally
underrepresented: control-state coupling and time-varying parameters. Our
bilinear extension addresses both with a minimal architectural change,
fewer than $1\%$ added parameters, exact model Jacobians, and an SCP
controller with monotone descent and KKT accumulation points under the
stated assumptions.

Across CartPole (which exercises $u \cdot x$ coupling) and RSCP (which
exercises time-varying parameters in the TV variant without $u \cdot x$
coupling), four consistent effects emerged: forecasting non-inferiority on
all cells under deployment-relevant averaging, with strict gains where
\Cref{prop:forecasting} predicts them; substantial training stabilization
on RSCP TV; closed-loop MPC gains on RSCP TV when SCP is iterated; and graceful degradation under stale-plan execution, where the
bilinear model maintains a clear robustness advantage. The learned
$\|G\|_F$ patterns provide an additional interpretive diagnostic for which
mechanism each cell primarily exercises.

Open directions include extending the bilinear parameterization to the
step-varying $B_k$ matrices, imposing structured priors on $G_i$ that encode
known physical coupling topology, and designing benchmarks that isolate the
$u \cdot x$ structural axis from time-varying parameters more cleanly.

\section*{AI Assistance Disclosure}
We used AI assistants (large language models) for language editing,
phrasing, and LaTeX formatting throughout this manuscript. All technical
content, mathematical results, experimental design, code, and analyses
are the authors' own work. The AI tools did not contribute novel ideas,
proofs, or empirical results.

\bibliographystyle{abbrvnat}
\bibliography{refs}

@article{koopman1931hamiltonian,
  author  = {Koopman, Bernard O.},
  title   = {Hamiltonian Systems and Transformation in {H}ilbert Space},
  journal = {Proceedings of the National Academy of Sciences},
  year    = {1931},
  volume  = {17},
  number  = {5},
  pages   = {315--318},
}

@article{mezic2005spectral,
  author  = {Mezi{\'c}, Igor},
  title   = {Spectral Properties of Dynamical Systems, Model Reduction
             and Decompositions},
  journal = {Nonlinear Dynamics},
  year    = {2005},
  volume  = {41},
  pages   = {309--325},
}

@article{korda2018linear,
  author  = {Korda, Milan and Mezi{\'c}, Igor},
  title   = {Linear Predictors for Nonlinear Dynamical Systems:
             {K}oopman Operator Meets Model Predictive Control},
  journal = {Automatica},
  year    = {2018},
  volume  = {93},
  pages   = {149--160},
}

@article{williams2015data,
  author  = {Williams, Matthew O. and Kevrekidis, Ioannis G. and
             Rowley, Clarence W.},
  title   = {A Data-Driven Approximation of the {K}oopman Operator:
             Extending Dynamic Mode Decomposition},
  journal = {Journal of Nonlinear Science},
  year    = {2015},
  volume  = {25},
  pages   = {1307--1346},
}

@article{lusch2018deep,
  author  = {Lusch, Bethany and Kutz, J. Nathan and Brunton, Steven L.},
  title   = {Deep Learning for Universal Linear Embeddings of Nonlinear
             Dynamics},
  journal = {Nature Communications},
  year    = {2018},
  volume  = {9},
  pages   = {4950},
}

@inproceedings{li2025mamko,
  author    = {Li, Zhaoyang and Han, Minghao and Yin, Xunyuan},
  title     = {{MamKO}: {M}amba-Based {K}oopman Operator for Modeling and
               Predictive Control},
  booktitle = {International Conference on Learning Representations ({ICLR})},
  year      = {2025},
}

@article{gu2023mamba,
  author  = {Gu, Albert and Dao, Tri},
  title   = {Mamba: Linear-Time Sequence Modeling with Selective State Spaces},
  journal = {arXiv preprint arXiv:2312.00752},
  year    = {2023},
}

@book{mohler1973bilinear,
  author    = {Mohler, Ronald R.},
  title     = {Bilinear Control Processes},
  publisher = {Academic Press},
  year      = {1973},
}

@book{Elliott2009bilinear,
  author    = {Elliott, David L.},
  title     = {Bilinear Control Systems: Matrices in Action},
  publisher = {Springer},
  year      = {2009},
}

@article{mao2018successive,
  author  = {Mao, Yuanqi and Szmuk, Michael and Xu, Xiangru and
             A\c{c}ikme\c{s}e, Beh\c{c}et},
  title   = {Successive Convexification: A Superlinearly Convergent Algorithm
             for Non-convex Optimal Control Problems},
  journal = {arXiv preprint arXiv:1804.06539},
  year    = {2018},
}

@article{malyuta2022convex,
  author  = {Malyuta, Danylo and Reynolds, Taylor P. and Szmuk, Michael and
             Lew, Thomas and Bonalli, Riccardo and Pavone, Marco and
             A\c{c}ikme\c{s}e, Beh\c{c}et},
  title   = {Convex Optimization for Trajectory Generation: A Tutorial on
             Generating Dynamically Feasible Trajectories Reliably and
             Efficiently},
  journal = {IEEE Control Systems Magazine},
  year    = {2022},
  volume  = {42},
  number  = {5},
  pages   = {40--113},
}

@article{li2004iterative,
  author  = {Li, Weiwei and Todorov, Emanuel},
  title   = {Iterative Linear Quadratic Regulator Design for Nonlinear
             Biological Movement Systems},
  journal = {ICINCO},
  year    = {2004},
  pages   = {222--229},
}

@article{stellato2020osqp,
  author  = {Stellato, Bartolomeo and Banjac, Goran and Goulart, Paul and
             Bemporad, Alberto and Boyd, Stephen},
  title   = {{OSQP}: An Operator Splitting Solver for Quadratic Programs},
  journal = {Mathematical Programming Computation},
  year    = {2020},
  volume  = {12},
  pages   = {637--672},
}

@inproceedings{han2021desko,
  author    = {Han, Minghao and Euler-Rolle, Jacob and Katzschmann, Robert K.},
  title     = {{DeSKO}: Stability-Assured Robust Control with a Deep Stochastic
               {K}oopman Operator},
  booktitle = {International Conference on Learning Representations ({ICLR})},
  year      = {2021},
}

@inproceedings{li2020compositional,
  author    = {Li, Yunzhu and He, Hao and Wu, Jiajun and Katabi, Dina and
               Torralba, Antonio},
  title     = {Learning Compositional {K}oopman Operators for Model-Based
               Control},
  booktitle = {International Conference on Learning Representations ({ICLR})},
  year      = {2020},
}

@article{hao2022deep,
  author  = {Hao, Wenjian and Huang, Bowen and Pan, Wei and Wu, Di and
             Mou, Shaoshuai},
  title   = {Deep {K}oopman Learning of Nonlinear Time-Varying Systems},
  journal = {Automatica},
  year    = {2024},
  volume  = {159},
  pages   = {111372},
}

@article{brunton2016sindy,
  author  = {Brunton, Steven L. and Proctor, Joshua L. and Kutz, J. Nathan},
  title   = {Discovering Governing Equations from Data by Sparse Identification
             of Nonlinear Dynamical Systems},
  journal = {Proceedings of the National Academy of Sciences},
  year    = {2016},
  volume  = {113},
  number  = {15},
  pages   = {3932--3937},
}

@article{kaiser2018sparse,
  author  = {Kaiser, Eurika and Kutz, J. Nathan and Brunton, Steven L.},
  title   = {Sparse Identification of Nonlinear Dynamics for Model Predictive
             Control in the Low-Data Limit},
  journal = {Proceedings of the Royal Society A},
  year    = {2018},
  volume  = {474},
  pages   = {20180335},
}

@inproceedings{chen2018neural,
  author    = {Chen, Ricky T.Q. and Rubanova, Yulia and Bettencourt, Jesse and
               Duvenaud, David},
  title     = {Neural Ordinary Differential Equations},
  booktitle = {Advances in Neural Information Processing Systems (NeurIPS)},
  year      = {2018},
  volume    = {31},
}

@article{liu2008distributed,
  author  = {Liu, Jinfeng and Mu{\~n}oz de la Pe{\~n}a, David and
             Ohran, Brett J. and Christofides, Panagiotis D. and
             Davis, James F.},
  title   = {A Two-Tier Architecture for Networked Process Control},
  journal = {Chemical Engineering Science},
  year    = {2008},
  volume  = {63},
  number  = {22},
  pages   = {5394--5409},
}

@article{liu2009DistributedMPC,
  author  = {Liu, Jinfeng and Mu{\~n}oz de la Pe{\~n}a, David and
             Christofides, Panagiotis D.},
  title   = {Distributed Model Predictive Control of Nonlinear Process
             Systems},
  journal = {AIChE Journal},
  year    = {2009},
  volume  = {55},
  number  = {5},
  pages   = {1171--1184},
}

@phdthesis{chen2022dissertation,
  author  = {Chen, Siyao},
  title   = {A Machine Learning-Based Approach to Cybersecurity and Safety
             of Model Predictive Control Systems},
  school  = {University of California, Los Angeles},
  year    = {2022},
  type    = {Ph.D. Dissertation},
  note    = {Section 2.4 documents the corrected RSCP formulation used in
             this work.},
}

@inproceedings{kingma2014adam,
  author    = {Kingma, Diederik P. and Ba, Jimmy},
  title     = {{A}dam: A Method for Stochastic Optimization},
  booktitle = {International Conference on Learning Representations ({ICLR})},
  year      = {2015},
}

@article{mclachlan2002splitting,
  title = {Splitting methods},
  author = {McLachlan, Robert I. and Quispel, G. Reinout W.},
  journal = {Acta Numerica},
  volume = {11},
  pages = {341--434},
  year = {2002},
  publisher = {Cambridge University Press},
  doi = {10.1017/S0962492902000053}
}

@article{trotter1959product,
  title = {On the product of semi-groups of operators},
  author = {Trotter, H. F.},
  journal = {Proceedings of the American Mathematical Society},
  volume = {10},
  number = {4},
  pages = {545--551},
  year = {1959},
  doi = {10.1090/S0002-9939-1959-0108732-6}
}

@inbook{goswami2020global,
  author    = {Goswami, Debdipta and Paley, Derek A.},
  editor    = {Mauroy, Alexandre and Mezi{\'c}, Igor and Susuki, Yoshihiko},
  title     = {Global Bilinearization and Reachability Analysis of Control-Affine Nonlinear Systems},
  booktitle = {The Koopman Operator in Systems and Control: Concepts, Methodologies, and Applications},
  year      = {2020},
  publisher = {Springer International Publishing},
  address   = {Cham},
  pages     = {81--98},
  isbn      = {978-3-030-35713-9},
  doi       = {10.1007/978-3-030-35713-9_4}
}

\appendix
 
\section{Stability: Training Penalty and Inference-Time Checks}
\label{app:stability}
 
\paragraph{Training.}
The spectral penalty in the loss is
\[
  \mathcal{L}_{\mathrm{stab}} = \sum_j \mathrm{ReLU}\bigl(|\lambda_j(A_{\mathrm{disc},k})| - 1 + m_s\bigr),
\]
with margin $m_s = 0.05$. Eigenvalues are computed via
\texttt{torch.linalg.}\allowbreak\texttt{eigvals} cast to \texttt{float32}
to avoid mixed-precision artifacts. Zero initialization of $G_i$ ensures the
model starts as an exact copy of the Linear baseline; the penalty is then
inactive whenever every diagonal mode satisfies $|e^{a_n \delta_n}| \leq 1 - m_s$.
 
\paragraph{Inference.}
We do not perform a per-step spectral check at inference. In practice,
the SCP trust-region constraint and the control bounds keep the
linearized dynamics within the regime where our trained operators are
well-conditioned; this is not a stability guarantee. We
did, however, evaluate the cheap Gershgorin disk pre-check
($|A_{\mathrm{disc},k,ii}| + \sum_{j \neq i}|A_{\mathrm{disc},k,ij}| < 1$)
as an $O(\dz^2)$ stability certificate at each MPC step. Empirically the
disks straddle the unit circle in $100\%$ of MPC steps across both
bilinear TV cells (CartPole and RSCP, $d \in \{0,1,3,5\}$, $10$ episodes
$\times 1000$ steps each), because the off-diagonal mass introduced by the
bilinear factor $E_P(u_k) = \exp(\sum_i u_k^{(i)} G_i)$ inflates the radii
past unity even when $\rho(A_{\mathrm{disc},k}) < 1$. A full
eigendecomposition would therefore be required for any runtime spectral
gate; for $\dz \leq 64$ this is sub-millisecond per step and does not
bottleneck the MPC loop, but we leave runtime stability gating to future
work.
 
\section{System Dynamics}
\label{app:systems}
 
This appendix specifies the governing ODEs, parameters, and action spaces of
the four system variants (CartPole TI/TV, RSCP TI/TV) used in the paper.
 
\subsection{CartPole (TI, TV)}
\label{app:cartpole}
 
The state is $(x, \dot{x}, \theta, \dot{\theta})$ with $\theta$ measured from
the upright. The control is the horizontal force $F$ applied to the cart.
Following the MamKO benchmark protocol~\citep{li2025mamko}, we use
gravity $g = 10$\,m\,s$^{-2}$, cart mass $m_c = 1$\,kg, pole mass
$m_p = 0.1$\,kg, half-length $\ell = 0.5$\,m, force bound
$|F| \leq 20$\,N, and sampling period $\Delta t = 0.02$\,s. The equations of
motion include cart and pole friction terms $\mu_c, \mu_p$ as in the MamKO
reference implementation.
 
\paragraph{Time-invariant variant (TI).}
The friction coefficients are fixed at $\mu_c = 0$, $\mu_p = 0$ (frictionless
limit). This recovers the standard frictionless CartPole.
 
\paragraph{Time-varying variant (TV).}
The cart friction is modulated as
\[
  \mu_c(t) = \mu_c^{\mathrm{base}} + \sin(\omega\, t),
\]
with $\mu_c^{\mathrm{base}} = 5\times10^{-4}$, $\mu_p = 2\times10^{-6}$, and
$\omega = 1$\,rad\,s$^{-1}$, matching the MamKO TV configuration. Since
$\mu_c(t)$ is negative for part of each cycle, it acts as a signed synthetic
friction modifier rather than a dissipative friction coefficient. This is
the specific cell used in the closed-loop MPC and lead-time experiments;
results for $\omega \in \{0.1, 10\}$ are not reported here.
 
\paragraph{Note on the sin/cos discrepancy.}
The MamKO paper writes $\mu_c(t) = \mu_c^{\mathrm{base}} + \cos(\omega t)$
(Eq.\ 16 of~\citet{li2025mamko}); the official MamKO repository uses
\texttt{sin} in \texttt{cartpole\_V.py}. We follow the repository, which is
the executable ground truth.

\paragraph{Parameters.}
\Cref{tab:cartpole_params} lists the numerical constants used for the
CartPole benchmark. The dynamical constants and time-varying friction
settings follow the MamKO benchmark protocol~\citep{li2025mamko}; the
shared MamKO architecture hyperparameters for CartPole are summarized
separately in \Cref{tab:arch}.

\begin{table*}[ht]
  \centering
  \caption{CartPole parameters for the time-invariant and time-varying
    benchmarks.}
  \label{tab:cartpole_params}
  \footnotesize
  \begin{tabular*}{\textwidth}{@{\extracolsep{\fill}} l l p{0.62\textwidth}}
    \toprule
    Symbol & Value & Description \\
    \midrule
    $g$ & $10$\,m\,s$^{-2}$ & gravitational acceleration \\
    $m_c$ & $1$\,kg & cart mass \\
    $m_p$ & $0.1$\,kg & pole mass \\
    $\ell$ & $0.5$\,m & pole half-length \\
    $\Delta t$ & $0.02$\,s & sampling period \\
    $|F|$ & $\leq 20$\,N & control bound on horizontal force \\
    $\mu_c^{\mathrm{base}}$ & $5{\times}10^{-4}$ & base cart friction coefficient setting \\
    $\mu_p$ & $2{\times}10^{-6}$ & pole friction coefficient\\
    $\omega$ (TV) & $1$\,rad\,s$^{-1}$ & frequency in $\mu_c(t) = \mu_c^{\mathrm{base}} + \sin(\omega t)$ \\
    \bottomrule
  \end{tabular*}
\end{table*}
 
\subsection{RSCP: Reactor--Separator Process}
\label{app:rscp}
 
RSCP is a reactor--separator process consisting of two continuous
stirred-tank reactors (CSTR-1 and CSTR-2) operating in series, feeding a
flash separator whose condensed overhead is partially recycled to
CSTR-1~\citep{liu2008distributed}. The system was adopted as the MamKO
benchmark in~\citet{li2025mamko}. Two consecutive first-order reactions,
$A \to B$ and $B \to C$, proceed in each CSTR with Arrhenius kinetics; the separator
performs ideal vapor--liquid equilibrium at fixed pressure with no reaction.
The state vector is
\[ x = (x_{A,1},\, x_{B,1},\, T_1,\, x_{A,2},\, x_{B,2},\, T_2,\, x_{A,3},\, x_{B,3},\, T_3), \]
where $x_{A,i}, x_{B,i}$ are the mass fractions of species $A, B$ in
vessel $i$ and $T_i$ is the corresponding temperature in Kelvin; vessel
indices $1, 2$ refer to the two CSTRs and $3$ to the flash separator. The
control is $u = (Q_1, Q_2, Q_3)$, the heat duties applied to each vessel in
kJ\,h$^{-1}$.
 
\paragraph{Governing ODEs.}
The mass and energy balances follow~\citet[Section 2.4]{chen2022dissertation}
rather than the formulation in the original Liu papers; see ``Reproducibility
notes'' below for the discrepancies that motivated this choice. Letting
$F_1 = F_{10} + F_r$ and $F_2 = F_1 + F_{20}$ denote the total throughput of
each CSTR, $r_{1,i} = k_1 \exp(-E_1/(RT_i))$ and
$r_{2,i} = k_2 \exp(-E_2/(RT_i))$ the Arrhenius rates, and
$\beta_j = (-\Delta H_j) C_M / (\rho c_p)$ the heat-of-reaction coefficients
(see paragraph below), the governing equations for CSTR-1 and CSTR-2 are
\begin{align*}
  \dot{x}_{A,1} &= \tfrac{F_{10}}{V_1}(x_{A,10} - x_{A,1})
                  + \tfrac{F_r}{V_1}(x_{A,r} - x_{A,1}) \\
                & \quad - r_{1,1} x_{A,1}, \\
  \dot{x}_{B,1} &= \tfrac{F_{10}}{V_1}(x_{B,10} - x_{B,1})
                  + \tfrac{F_r}{V_1}(x_{B,r} - x_{B,1}) \\
                & \quad + r_{1,1} x_{A,1} - r_{2,1} x_{B,1}, \\
   \dot{T}_1     &= \tfrac{F_{10}}{V_1}(T_{10} - T_1)
                   + \tfrac{F_r}{V_1}(T_3 - T_1) \\
                 & \quad + \beta_1 r_{1,1} x_{A,1}
                        + \beta_2 r_{2,1} x_{B,1} \\
                 & \quad + \tfrac{Q_1}{\rho c_p V_1},
\end{align*}
and analogously for CSTR-2 with $(F_{20}, F_1, V_2, T_{20}, x_{\cdot,20}, Q_2)$
replacing $(F_{10}, F_r, V_1, T_{10}, x_{\cdot,10}, Q_1)$ and the CSTR-1 state
replacing the recycle stream. The flash separator carries no reaction:
\begin{align*}
  \dot{x}_{A,3} &= \tfrac{F_2}{V_3}(x_{A,2} - x_{A,3})
                  - \tfrac{F_r + F_p}{V_3}(x_{A,r} - x_{A,3}), \\
  \dot{x}_{B,3} &= \tfrac{F_2}{V_3}(x_{B,2} - x_{B,3})
                  - \tfrac{F_r + F_p}{V_3}(x_{B,r} - x_{B,3}), \\
   \dot{T}_3     &= \tfrac{F_2}{V_3}(T_2 - T_3) + \tfrac{Q_3}{\rho c_p V_3} \\
                 & \quad + \tfrac{(F_r + F_p) C_M}{\rho c_p V_3}
                   \bigl(x_{A,r} \Delta H_{\mathrm{vap},A} \\
                 & \quad\quad + x_{B,r} \Delta H_{\mathrm{vap},B}
                              + x_{C,r} \Delta H_{\mathrm{vap},C}\bigr).
\end{align*}
The recycle compositions $x_{A,r}, x_{B,r}, x_{C,r}$ are determined by ideal
vapor-liquid equilibrium with relative volatilities $(\alpha_A, \alpha_B,
\alpha_C)$,
\begin{equation}
\begin{aligned}
  x_{A,r} &= \tfrac{\alpha_A x_{A,3}}{D}, \\
  x_{B,r} &= \tfrac{\alpha_B x_{B,3}}{D}, \\
  x_{C,r} &= \tfrac{\alpha_C (1 - x_{A,3} - x_{B,3})}{D}.
\end{aligned}
\end{equation}
\[
  D = \alpha_A x_{A,3} + \alpha_B x_{B,3}
    + \alpha_C (1 - x_{A,3} - x_{B,3}).
\]
The structurally important property for this paper is that the
heat duties $Q_i$ enter the energy balances \emph{additively} as
$+Q_i / (\rho c_p V_i)$ with constant coefficients, contributing no
$u \cdot x$ terms.
 
\paragraph{Parameters.}
\Cref{tab:rscp_params} lists the numerical values used. Volumes, flows, feed
states, kinetic constants, activation energies, heats of reaction, and
relative volatilities are taken from~\citet[Table~2.1]{chen2022dissertation};
the molar concentration factor $C_M = 2$\,kmol\,m$^{-3}$ is the constant
that converts mass fractions to molar concentrations in the energy balance
(see ``Reproducibility notes'').
 
\begin{table*}[ht]
  \centering
  \caption{RSCP parameters (after the corrections in
    \Cref{app:rscp_repro}).}
  \label{tab:rscp_params}
  \footnotesize
  \begin{tabular}{l l p{0.45\columnwidth}}
    \toprule
    Symbol & Value & Description \\
    \midrule
    $V_1, V_2, V_3$ & $1.0,\ 0.5,\ 1.0$\,m$^3$ & vessel volumes \\
    $F_{10}, F_{20}$ & $5.04,\ 5.04$\,m$^3$\,h$^{-1}$ & feed flows \\
    $F_r, F_p$ & $50.4,\ 5.04$\,m$^3$\,h$^{-1}$ & recycle, purge \\
    $T_{10}, T_{20}$ & $300,\ 300$\,K & feed temperatures \\
    $x_{A,10}, x_{A,20}$ & $1.0,\ 1.0$ & feed compositions ($A$) \\
    $k_1, k_2$ & $9.972{\times}10^6,\ 9.36{\times}10^6$\,h$^{-1}$ & kinetic constants \\
    $E_1, E_2$ & $5{\times}10^4,\ 6{\times}10^4$\,kJ\,kmol$^{-1}$ & activation energies \\
    $\Delta H_1, \Delta H_2$ & $-1.2{\times}10^5,\ -1.4{\times}10^5$\,kJ\,kmol$^{-1}$ & reaction enthalpies \\
    $\rho$ & $1000$\,kg\,m$^{-3}$ & density \\
    $c_p$ & $4.2$\,kJ\,kg$^{-1}$\,K$^{-1}$ & specific heat \\
    $C_M$ & $2$\,kmol\,m$^{-3}$ & molar concentration factor \\
    $\Delta H_{\mathrm{vap},A,B,C}$ & $-3.53,\ -1.57,\ -4.07$\,$\times 10^4$\,kJ\,kmol$^{-1}$ & vap.\ enthalpies \\
    $\alpha_A, \alpha_B, \alpha_C$ & $3.5,\ 1.0,\ 0.5$ & relative volatilities \\
    $R$ & $8.314$\,kJ\,kmol$^{-1}$\,K$^{-1}$ & gas constant \\
    \bottomrule
  \end{tabular}
\end{table*}
 
\paragraph{Reproducibility notes.}
\label{app:rscp_repro}
We verified that the RSCP ODE system as written in~\citet{liu2008distributed,
liu2009DistributedMPC} does not reproduce the steady state reported by
those papers under their stated parameter values. The discrepancy resolves
under three corrections drawn from~\citet[Section~2.4]{chen2022dissertation}:
(i) the heat-of-reaction term carries a molar concentration factor
$\beta_j = (-\Delta H_j) C_M / (\rho c_p)$ rather than $\Delta H_j / c_p$,
where $C_M = 2$\,kmol\,m$^{-3}$; (ii) the separator energy balance
includes the convective transport of vaporization enthalpies
$\Delta H_{\mathrm{vap},\{A,B,C\}}$ for the three species;
(iii) the steady-state heat duties $Q_i$ are of order $10^6$\,kJ\,h$^{-1}$
rather than the $10^9$ value listed in~\citet[Table~2.1]{chen2022dissertation},
which we identified as a typographical exponent error. Under the corrected
formulation, forward integration from the published steady state has
$\|\dot{x}\|_\infty < 10^{-6}$ on temperatures and within $4 \times 10^{-3}$
on compositions. We use this corrected system throughout. The published MamKO
implementation of RSCP is not publicly available, so we cannot directly
confirm that our reconstructed system matches theirs; however, the MamKO
paper's reported steady state and our forward-integration check are
mutually consistent under the corrected ODEs.
 
\paragraph{Time-varying variant.}
The TV variant follows the MamKO synthetic time-varying modifier
$\varphi_c(t) = e^{-0.01 t}$~\citep[Appendix C.3]{li2025mamko}, with $t$ in
seconds, applied multiplicatively to both Arrhenius rate terms in the two CSTRs. This
simulates catalyst deactivation: reaction rates decay exponentially with
time, and the model must track operating-point drift induced by the slowly
shrinking effective kinetics. Crucially, $\varphi_c(t)$ multiplies only the
kinetic terms, not the convective or heat-duty terms, so the structural
property that $u_k$ does not couple multiplicatively with $x_k$ is preserved
in the TV variant: RSCP TV exercises the time-varying-parameter axis of
\Cref{sec:intro} without introducing $u \cdot x$ structure.
 
\paragraph{Nominal operating point and controls.}
The MPC tracking task drives the system to the nominal steady state
\begin{multline*}
  x_s = (0.18,\; 0.67,\; 480.32,\; 0.20,\; 0.65, \\
         472.79,\; 0.07,\; 0.67,\; 474.89)
\end{multline*}
following~\citet[Appendix C.3]{li2025mamko}. The control bounds are
$|Q_i - Q_{i,s}| \leq 10^6$\,kJ\,h$^{-1}$ around the nominal heat duties
$Q_{i,s}$; these are enforced by the simulator's action-space clipping
rather than by the QP, since the QP operates in normalized control space
(see~\Cref{app:scp}). The sampling period is $\Delta t = 18$\,s
($= 0.005$\,h).
 
\section{Data Generation and Training Protocol}
\label{app:training}
 
\subsection{Trajectory Generation}
\label{app:data}
 
For each system variant we generate training and test data with the
official MamKO \texttt{ReplayMemory} pipeline. Continuous rollouts under
per-step i.i.d. uniform random control excitation are segmented into
overlapping input--output windows of length
$H_{\mathrm{lookback}} + H_{\mathrm{pred}} = 60$. Episodes are integrated by
forward Euler at the system sampling period, with no sub-stepping:
$\Delta t = 0.02$\,s for CartPole and $\Delta t = 18$\,s for RSCP.
Episodes terminate on physical bound violations or at a fixed horizon,
whichever comes first: for CartPole, when $|\theta| > 20^\circ$ or $|x| > 10$;
for RSCP, when concentration or temperature bounds are violated; for both
systems, at $20{,}040$ steps for training episodes and $1{,}000$ steps for
test episodes. Windows are pooled across episodes until the configured
target is reached: $39{,}900$ training windows, split $80/20$ into
train/validation subsets via a random split with seed $1$, and $4{,}000$
test windows.

Control excitation is drawn independently at each step. For CartPole, the
force is sampled as $F \sim \mathcal{U}(-20, 20)$\,N. For RSCP, each heat
duty is sampled as
$Q_i \sim \mathcal{U}(Q_i^s - 10^6,\ Q_i^s + 10^6)$\,kJ\,h$^{-1}$, where
$Q^s$ denotes the steady-state heating rates. Initial states are sampled at
the start of each episode: for CartPole,
$x_0 \sim \mathcal{U}(-4, 4)$\,m and
$\theta_0 \sim \mathcal{U}(-0.1, 0.1)$\,rad, with zero initial velocities;
for RSCP, components are drawn independently and uniformly within
per-component half-widths around a verified fixed point $\bar{x}$ of the
ODEs,
$x_0 \sim \mathcal{U}(\bar{x} - \boldsymbol{\rho},\ \bar{x} + \boldsymbol{\rho})$
with
$\boldsymbol{\rho} = (0.05,\ 0.05,\ 10\,\mathrm{K},\ 0.05,\ 0.05,\ 10\,\mathrm{K},\ 0.02,\ 0.05,\ 10\,\mathrm{K})$
in the state ordering $(x_{A,i}, x_{B,i}, T_i)$ for $i = 1, 2, 3$. Here
$\bar{x}$ denotes the numerically verified fixed point of the corrected RSCP
ODEs rather than the approximate steady state reported in the MamKO
appendix; we use $\bar{x}$ as the perturbation center to avoid systematic
drift back toward the true equilibrium during data generation.
 
\subsection{Model Architecture}
\label{app:arch}
 
The encoder, backbone, and decoder follow the MamKO
configuration~\citep{li2025mamko} for each system class. We re-use the
MamKO reference implementation hyperparameters; only the bilinear extension
introduces new parameters ($G_i$ tensors at full rank $r = d_z$, plus the
spectral penalty weight $\lambda_s$). \Cref{tab:arch} summarizes the per-system
settings.
 
\begin{table}[ht]
  \centering
  \caption{Architecture hyperparameters per system. CartPole settings follow
    the MamKO reference implementation; RSCP settings follow its RSCP
    benchmark configuration. Bilinear-specific entries ($r$, $\lambda_s$,
    discretization scheme) are reported under \emph{Bilinear-only}.}
  \label{tab:arch}
  \footnotesize
  \begin{tabular}{lrr}
    \toprule
                                  & \textbf{CartPole} & \textbf{RSCP} \\
    \midrule
    Lookback window               & $30$ & $30$ \\
    Forecast horizon              & $30$ & $30$ \\
    Latent dimension $\dz$        & $8$ & $15$ \\
    Mamba conv kernel $d_{\mathrm{conv}}$ & $15$ & $5$ \\
    Hidden dimension              & $64$ & $64$ \\
    \midrule
    \multicolumn{3}{l}{\emph{Bilinear-only}} \\
    Bilinear rank $r$             & $8$ & $15$ \\
    Spectral penalty $\lambda_s$  & $0.01$ & $0.01$ \\
    Discretization                & \texttt{lie\_trotter} & \texttt{lie\_trotter} \\
    \bottomrule
  \end{tabular}
\end{table}
 
In all reported runs ($r = \dz$) the $G_i$ tensors are parameterized
directly and randomly initialized; the optional factorization
$G_i = L_i R_i^\top$ with $L_i, R_i \in \R^{\dz \times r}$ initializes $L_i$
with small random values and $R_i$ at zero. Zero initialization of $G_i$
makes the bilinear model an exact copy of the Linear baseline at training step zero, so any
training-time divergence between the two is attributable to the bilinear
extension.
 
\subsection{Lie--Trotter Discretization Details}
\label{app:lie_trotter_details}

The baseline MamKO diagonal step is preserved exactly by evaluating the
diagonal flow element-wise:
\[
  [E_D]_{nn} = \exp(a_n \delta_n).
\]
The associated diagonal input integral is computed element-wise as
\begin{equation}
  [\bar{B}_{\mathrm{diag},k}]_{n,j}
  = \frac{\exp(a_n \delta_n) - 1}{a_n}\,[B_k^{(c)}]_{n,j}.
  \label{eq:bdiag}
\end{equation}
In implementation, the numerator is evaluated with \texttt{torch.expm1} to
avoid catastrophic cancellation as $a_n \to 0$, with a small sign-preserving
denominator offset used only for numerical stability.

The coupling factor $E_P(u_k)$ is evaluated with a single $\dz \times \dz$
\texttt{torch.linalg.matrix\_exp} call, yielding the discrete-time matrices
\begin{equation}
\begin{aligned}
  A_{\mathrm{disc},k} &= E_P(u_k) E_D, \\
  B_{\mathrm{disc},k} &= E_P(u_k) \, \bar{B}_{\mathrm{diag},k}.
\end{aligned}
\label{eq:lie_trotter_disc}
\end{equation}
Thus the diagonal ZOH step costs $O(\dz)$ and the coupling step costs
$O(\dz^3)$, rather than the augmented $(\dz + m) \times (\dz + m)$ matrix
exponential required by a dense ZOH implementation that jointly recovers
$A_{\mathrm{disc},k}$ and $B_{\mathrm{disc},k}$. For $\dz \leq 64$, this
per-step discretization overhead remains negligible relative to backbone
inference.

\subsection{Optimization}
\label{app:optim}
 
All models are trained with Adam \citep{kingma2014adam} for
$401$ epochs at initial learning rate $1\mathrm{e}{-}3$,
weight decay $1\mathrm{e}{-}3$, batch size $256$, and a step learning-rate schedule (step size $50$ epochs, $\gamma = 0.9$). Gradient
clipping is applied at $1.0$. Best-validation-loss checkpoints
are retained for evaluation. Due to compute constraints, results are reported from a single seed per
configuration; per-cell dispersion is reported across the 10 closed-loop
episodes and across the final 50 training epochs (mean$_{50}$ in \Cref{tab:forecast}). Multi-seed validation is deferred to a follow-up. Training
ran in FP32 throughout; mixed precision was not used.
 
The training loss is the multi-step observation-space prediction loss of
\Cref{eq:loss}, with the spectral penalty $\lambda_s
\mathcal{L}_{\mathrm{stab}}$ disabled at inference.
 
\subsection{Hardware and Wall-Clock}
\label{app:hardware}
 
Training runs were executed on NVIDIA A100 GPUs on GCE
\texttt{a2-ultragpu-1g} instances. Per-run training wall-clock was
approximately $4.8$ hours for CartPole and $9.75$ hours for RSCP
(bilinear Lie--Trotter models; corresponding linear baselines train roughly
$7\times$ (CartPole) and $14\times$ (RSCP) faster, at $\sim\!0.7$ hours per run on either system). Total
compute budget for the experiments reported in this paper, including
ablations and the lead-time sweep, was approximately $58$ GPU-hours
(training: $32$\,h across $8$ runs; closed-loop MPC
\texttt{regen=never} d-sweeps: $26$\,h across $4$ sweeps).
 
\section{Closed-Loop MPC Tasks}
\label{app:mpc}
 
This appendix specifies the closed-loop MPC problems evaluated in
\Cref{sec:tv_capture,sec:other_cells}: cost functions, control bounds,
warmup protocol, lead-time scheduling, and SCP hyperparameters.
 
\subsection{Common Setup}
\label{app:mpc_common}
 
For all systems we use prediction horizon $H_{\mathrm{MPC}} = 30$, apply
horizon $1$ (standard receding-horizon MPC at $d = 0$), and $10$ episodes
per configuration. Each episode is $1{,}000$ MPC steps. Following
MamKO Eq.\ 10a, the receding-horizon objective at time $k$ is
\[
\hat x_{j|k} := C_{j|k}\hat z_{j|k},
\]
\begin{equation}
\begin{aligned}
  J_k
  &= \sum_{j=k+1}^{k+H-1}
    \Bigl(
      \|\hat x_{j|k} - x_{\mathrm{ref}}\|_Q^2
      + \|u_{j|k} - u_{j-1|k}\|_R^2
    \Bigr) \\
  &\quad + \|\hat x_{k+H|k} - x_{\mathrm{ref}}\|_P^2.
\end{aligned}
\label{eq:mpc_cost_appendix}
\end{equation}
that is, a quadratic tracking term on the predicted decoded state together
with a $\Delta u$ smoothness penalty, plus a terminal tracking cost. The
QP is built by dynamics elimination over the $H_{\mathrm{MPC}}$-step
horizon and solved by OSQP~\citep{stellato2020osqp}.
 
\paragraph{Warmup.}
Following the MamKO evaluation protocol, the lookback buffer is initialized
synthetically by repeating the simulator's reset state $H$ times with zero
control. MPC begins at simulator step $t = 0$; no real warmup rollout is
performed. This convention matches MamKO's published evaluation
loop~\citep{li2025mamko} and is the operating point under which the
forecasting and closed-loop numbers in the main text are reported.
 
\subsection{CartPole TV}
\label{app:mpc_cp}
 
Reference: upright stabilization at $x_{\mathrm{ref}} = (0, 0, 0, 0)$ under
sinusoidal cart friction (\Cref{app:cartpole}). Stage and terminal cost
weights are
\begin{align*}
  Q &= \mathrm{diag}(1.0,\ 0.01,\ 100.0,\ 0.01), \\
  R &= 0.5, \\
  P &= \mathrm{diag}(5000,\ 0,\ 0,\ 0),
\end{align*}
matching the MamKO reference configuration~\citep{li2025mamko}. The heavy
weighting on $\theta$ (state index 3) reflects the priority of pole-angle
stabilization over cart-position regulation, and the terminal cost
penalizes only cart position. The control bound is $|F| \leq 20$\,N. There
is no safety constraint; the trajectory is unconstrained over the action
range.
 
\subsection{RSCP (TI, TV)}
\label{app:mpc_rscp}
 
Reference: setpoint tracking to the nominal steady state $x_s$ of
\Cref{app:rscp}. The cost weights heavily emphasize composition tracking
over temperature tracking:
\begin{align*}
  Q &= \mathrm{diag}(10^4,\ 10^4,\ 1,\ 10^4,\ 10^4,\ 1,\ 10^4,\ 10^4,\ 1), \\
  R &= 5\times10^{-12} \cdot I_3, \\
  P &= Q.
\end{align*}
The asymmetric weighting follows the MamKO RSCP configuration: the relevant
operational quantities are the species mass fractions, while the
temperature states float to whatever the heat duties drive them to. The
small $R$ value reflects the magnitude of the heat duties ($u \sim 10^6$
in raw units): $R = 5\times10^{-12}$ balances the
$\|\Delta u\|^2 \sim 10^{12}$ contribution against the
$\|x - x_s\|^2 \sim O(1)$ tracking term. Control bounds are
$|Q_i - Q_{i,s}| \leq 10^6$\,kJ\,h$^{-1}$ around the nominal duties; these
are enforced as action-space clipping in the simulator after the QP
returns its result in normalized control space (see~\Cref{app:scp}).
 
\subsection{SCP Hyperparameters}
\label{app:scp}
 
The SCP controller of Algorithm~\ref{alg:scp} runs with initial trust-region
radius $\varepsilon_0 = 1.0$ (in normalized control space), shrinkage
factor $1/2$ on cost-increase steps, and maximum SCP iterations
$N_{\mathrm{SCP}} \in \{1, 5\}$ as reported in \Cref{sec:tv_capture}. The
inner QP is solved by OSQP~\citep{stellato2020osqp} with default tolerances,
warm-started from the previous receding-horizon call's solution to amortize
solver iterations across the closed-loop trajectory. The QP is built in
\emph{normalized} control space using the MamKO instance-normalization
statistics output by the dynamics-generation network at each MPC step;
control bounds and the reference state are normalized correspondingly. After
the QP returns, the optimal control is denormalized and clipped to the
simulator's raw-space action bounds before application.
 
\subsection{Lead-Time Sweep}
\label{app:leadtime}
 
The lead-time experiment of \Cref{sec:tv_capture} sweeps the commitment
window $d \in \{0, 1, 3, 5\}$ on both RSCP TV and CartPole TV with
$N_{\mathrm{SCP}} = 5$ and $10$ episodes per $d$. The protocol is implemented by holding both the
optimizer's plan and the backbone's dynamics output fixed during the lead
window: at each MPC call the controller commits to the next $d$ controls
from a queue, applies the next queued control to the simulator, and refills
the queue from the freshly solved plan only after the queue empties; during
this window the backbone is \emph{not} re-evaluated, so the $\bar{A}_k$,
$\bar{B}_k$ matrices used for any auxiliary roll-forward are the ones
generated at the start of the window. We refer to this as
\texttt{regen=never} in our implementation; it isolates the
open-loop correction capacity of the model class itself, with neither model
receiving intermediate sensor or backbone updates within the lead window.
The $d = 10$ point is generated but excluded from the analysis in the main
text on grounds of discretization-saturation: at the RSCP sampling period
of $18$\,s, $d = 10$ corresponds to three minutes of stale plan, beyond the
regime in which the underlying linearization is informative (cumulative
costs converge for both models). The exact
end-of-horizon values reported in \Cref{fig:lead_time_rscp,fig:lead_time_cartpole}
are tabulated in \Cref{tab:lead_time_full,tab:lead_time_cartpole_full} for
RSCP TV and CartPole TV, respectively.
 
\begin{table*}[!t]
  \centering
  \caption{Cumulative log-cost at end of horizon, RSCP TV, lead-time sweep
    under \texttt{regen=never} with $N_{\mathrm{SCP}} = 5$ and
    $10$ episodes per cell. Lower is better.}
  \label{tab:lead_time_full}
  \begin{tabular}{lrrrr}
    \toprule
    & $d = 0$ & $d = 1$ & $d = 3$ & $d = 5$ \\
    \midrule
    Linear    & $-2.3233$ & $-1.8717$ & $-2.0559$ & $-2.2037$ \\
    Bilinear-SCP-5 & $-2.5442$ & $-2.7764$ & $-2.5662$ & $-2.4252$ \\
    \bottomrule
  \end{tabular}
\end{table*}
\vspace{-0.75em}
\begin{table*}[!t]
  \centering
  \caption{Cumulative log-cost at end of horizon, CartPole TV, lead-time sweep
    under \texttt{regen=never} with $N_{\mathrm{SCP}} = 5$ and
    $10$ episodes per cell. Lower is better.}
  \label{tab:lead_time_cartpole_full}
  \begin{tabular}{lrrrr}
    \toprule
    & $d = 0$ & $d = 1$ & $d = 3$ & $d = 5$ \\
    \midrule
    Linear    & $-1.9119$ & $-1.8033$ & $-1.6527$ & $-1.4707$ \\
    Bilinear-SCP-5 & $-1.9450$ & $-1.9086$ & $-1.7219$ & $-1.4972$ \\
    \bottomrule
  \end{tabular}
\end{table*}
\vspace{-0.75em}
\begin{table*}[!t]
  \centering
  \caption{Learned bilinear coupling norm $\|G\|_F$ at training convergence,
    per cell, with the qualitative shape of the training trajectory.}
  \label{tab:gi_norms}
  \footnotesize
  \begin{minipage}{0.72\textwidth}
    \centering
    \setlength{\tabcolsep}{8pt}
    \begin{tabular*}{\linewidth}{@{\extracolsep{\fill}}lcl}
      \toprule
      \textbf{Cell} & $\|G\|_F$ & \textbf{Trajectory shape} \\
      \midrule
      CartPole TI & $0.36$ & Peak $0.48$ at ep.\ $60$, decays \\
      CartPole TV & $0.57$ & Monotonic growth, recruitment \\
      RSCP TI     & $0.76$ & Plateau, slack absorption \\
      RSCP TV     & $0.44$ & Monotonic decay, multi-context \\
      \bottomrule
    \end{tabular*}
  \end{minipage}
\end{table*}

\section{Coupling Strength Diagnostic}
\label{app:gi_norms}

\Cref{tab:gi_norms} reports the learned $\|G\|_F$ at training convergence
across the four cells, alongside the qualitative shape of the training
trajectory. The four cells exhibit distinct training trajectories that
together provide an interpretive scaffold for the empirical results in the
main text.

The CartPole TI trajectory rises early to $\|G\|_F \approx 0.48$, then
decays under spectral regularization, settling at $0.36$. The bilinear
capacity is genuinely used---consistent with the strict forecasting gain
on this cell---but constrained from over-recruitment.

CartPole TV grows monotonically to $0.57$. The progressive recruitment
matches the time-varying nature of the perturbed dynamics: as training
proceeds, the model leans on $G$ to absorb friction-coefficient drift
that the linear lift cannot track within a single MPC horizon.

RSCP TI plateaus high at $0.76$. Forecast loss on this cell is tied with
Linear, while closed-loop cost is lower (\Cref{tab:mpc_cost}). We read
the norm as partly residual training slack: $u \cdot x$ is absent from the
physical ODE, although a nonlinear lift can still introduce latent coupling,
consistent with the closed-loop gain (\Cref{sec:other_cells}). \emph{$\|G\|_F$
alone should not be read as evidence of bilinear strength}: its size on
RSCP TI reflects training-time slack as well as useful coupling.

RSCP TV decays monotonically from an early peak to $0.44$. Under
time-varying parameters, $G$ must serve simultaneously across the family
of operating regimes induced by catalyst deactivation; the optimizer responds
by shrinking $G$ toward a setting that works on average rather than
overfitting to any single regime. This is consistent
with~\Cref{sec:tv_capture}: Bilinear-SCP-$5$, which iteratively
re-linearizes the operator at the current operating point, outperforms
Bilinear-SCP-$1$ on this cell because the average $G$ is what the model
carries while iterative SCP is what extracts operating-point-specific
use of it.

The four trajectories provide an interpretive scaffold rather than a
predictive theory: $\|G\|_F$ is consistent with the structural
condition in \Cref{prop:forecasting} (CartPole TI), with recruitment
under TV physics (CartPole TV), with averaging across operating regimes
(RSCP TV), and with slack absorption where the physical ODE has no
$u \cdot x$ term (RSCP TI).
The proposition does not predict trajectory shape; the trajectories make
sense in light of it.

\end{document}